\documentclass[twocolumn,journal]{IEEEtran}
\usepackage{amssymb}
\usepackage{mathrsfs}
\usepackage{amsfonts}
\usepackage{amsmath}
\usepackage{graphicx}
\usepackage{multirow}
\usepackage{caption2}
\usepackage{float}
\usepackage{booktabs}
\usepackage{algorithm}
\usepackage{algorithmic}
\usepackage{subfigure}
\usepackage{psfrag}
\usepackage{txfonts}
\usepackage{latexsym,bm}
\usepackage{color}
\usepackage[table]{xcolor}
\usepackage{url}
\UseRawInputEncoding
\newtheorem{theorem}{Theorem}
\newtheorem{corollary}{Corollary}
\newtheorem{lemma}{Lemma}

\newtheorem{proposition}{Proposition}

\newtheorem{definition}{Definition}

\newenvironment{proof}[1][Proof]{\textbf{#1. }}{\ \rule{0.5em}{0.5em}}%

\ifCLASSINFOpdf
\else
\fi

\begin{document}

\title{Distributed  Learning with Dependent Samples}

\author{Zirui Sun and Shao-Bo Lin
\IEEEcompsocitemizethanks{\IEEEcompsocthanksitem  Z. Sun and S. B. Lin  are with the Center of Intelligent Decision-Making and Machine Learning, School of Management, Xi'an Jiaotong University.  (Corresponding author: Shao-Bo Lin, Email: sblin1983@gmail.com)}}

\IEEEcompsoctitleabstractindextext{%
\begin{abstract}
This paper focuses on learning rate analysis of distributed kernel ridge regression (DKRR) for strong mixing sequences. Using a recently developed integral operator approach and a classical covariance inequality for Banach-valued strong mixing sequences, we succeed in deriving optimal learning rates  of DKRR. As a byproduct, we   deduce a sufficient condition for the mixing property to guarantee the  optimal learning rates for kernel ridge regression, which fills the gap of learning rates between i.i.d. samples and strong mixing sequences. A series of numerical experiments are conducted to verify our theoretical assertions via showing excellent learning performance of DKRR in learning both toy and real world time series data.
All these results extend the applicable range of distributed learning
from i.i.d. samples  to non-i.i.d. sequences.
\end{abstract}


\begin{IEEEkeywords}
Distributed learning,  strong mixing sequences, kernel ridge regression,
learning rate.
\end{IEEEkeywords}}

\maketitle

\IEEEdisplaynotcompsoctitleabstractindextext

\IEEEpeerreviewmaketitle
\section{Introduction}
With the development of data mining, data of massive size are collected in numerous application areas including   recommender systems, medical analysis, search engines, financial analysis, online text  sensor network monitoring  and social activity mining.
%
%
%
These massive data certainly bring    benefits for data analysis in terms of improving the prediction capability \cite{Zhang2015}, discovering potential structure of data which cannot be reflected by data of small size \cite{Chui2019}, and creating new growth opportunities to combine and analyze industry data \cite{Ban2019}. However, they also bring  several challenges, called massive data challenges, in data analysis  as follows:

$\bullet$ Distributive storage:   Chunks of data are spread across multiple possibly distant servers.
The cost  of data communications between different servers and risk of loading data in a remote server are relatively large. This makes the classical batch learning schemes such as kernel methods \cite{Steinwart2008} and neural networks \cite{Gyorfi2002} be not available. For example, Google has seen 30 trillion URLs, crawls over 20 billion of those a day \cite{Sullivan2012}, and answers 100 billion search queries a month. It is impossible to store all these data in a serve and  answer the query swiftly by using all these data in a batch manner.

 %
%
%
%

$\bullet$ Data privacy preservation: Individual  data such as clinical records, personal social networks and  financial fraud detection \cite{Horvitz2015} are frequently sensitive and cannot be shared for the sake of privacy. It is difficult  to analyze these  data   in a batch manner, hoping that all   sensitive data can be successfully  collected.
The data privacy issue  becomes  urgent when   data are owned by different organizations that wish to collaboratively use them. 
For instance, multiple hospitals are highly desired   to collaboratively train a classifier to improve the quality of clinical decision-making, in the premise of preserving the privacy of  their own medical records \cite{Fung2010}.

$\bullet$ Dependence for samples:
Massive data are collected asynchronously throughout time and behave as a stream or a collection of streams in many applications such as medical research and revenue management. Therefore, it is unreasonable to assume  them to be collected via an independent and identical (i.i.d.) manner. Under this circumstance, the classical learning approaches   \cite{Gyorfi2002,Cucker2007,Steinwart2008} developed for i.i.d. samples are no more efficient. For example, WTI (Western Texas Intermediate) Spot Prices from EIA (Energy Information Administration), an important part of the international energy pricing system \cite{Wang2011}, collect numerous  spot  prices as time series data which obviously do not satisfy the i.i.d. assumption. Analysis of massive WTI data is of great significance for economic development, enterprise production operations and investment.

Distributed learning \cite{Bertsekas1989} is a natural and preferable approach to conquer  massive data challenges. Different from the popular decentralized setting \cite{Boyd2005,Olfati-Saber2007,Nedic2009,Koppel2018}, we are interested in distributed learning with a global coordination (global machine) among servers (local machines) \cite{Zhang2013,Duchi2014,Lee2017,Jordan2018,Volgushev2019}. In such a setting,  data subsets are stored on numerous local machines and some specific learning algorithm is implemented on each local machine  to yield a local estimator. Then,  communications between  different local machines are conducted to exchange information  to improve the quality of local estimators. Finally, as shown in Figure 1, all the obtained local estimators are transmitted   to a global machine to produce a global estimator.  Therefore, there are  three ingredients of the
distributed learning:   local processing, communication, and synthesization. Such a distributed learning is a special type  of    the well known map-reduce approach \cite{Dean2008} and can be easily realized via the well developed Hadoop system \cite{White2012}.
\begin{figure}\label{Fig:1}
\begin{center}
\includegraphics[scale=0.3]{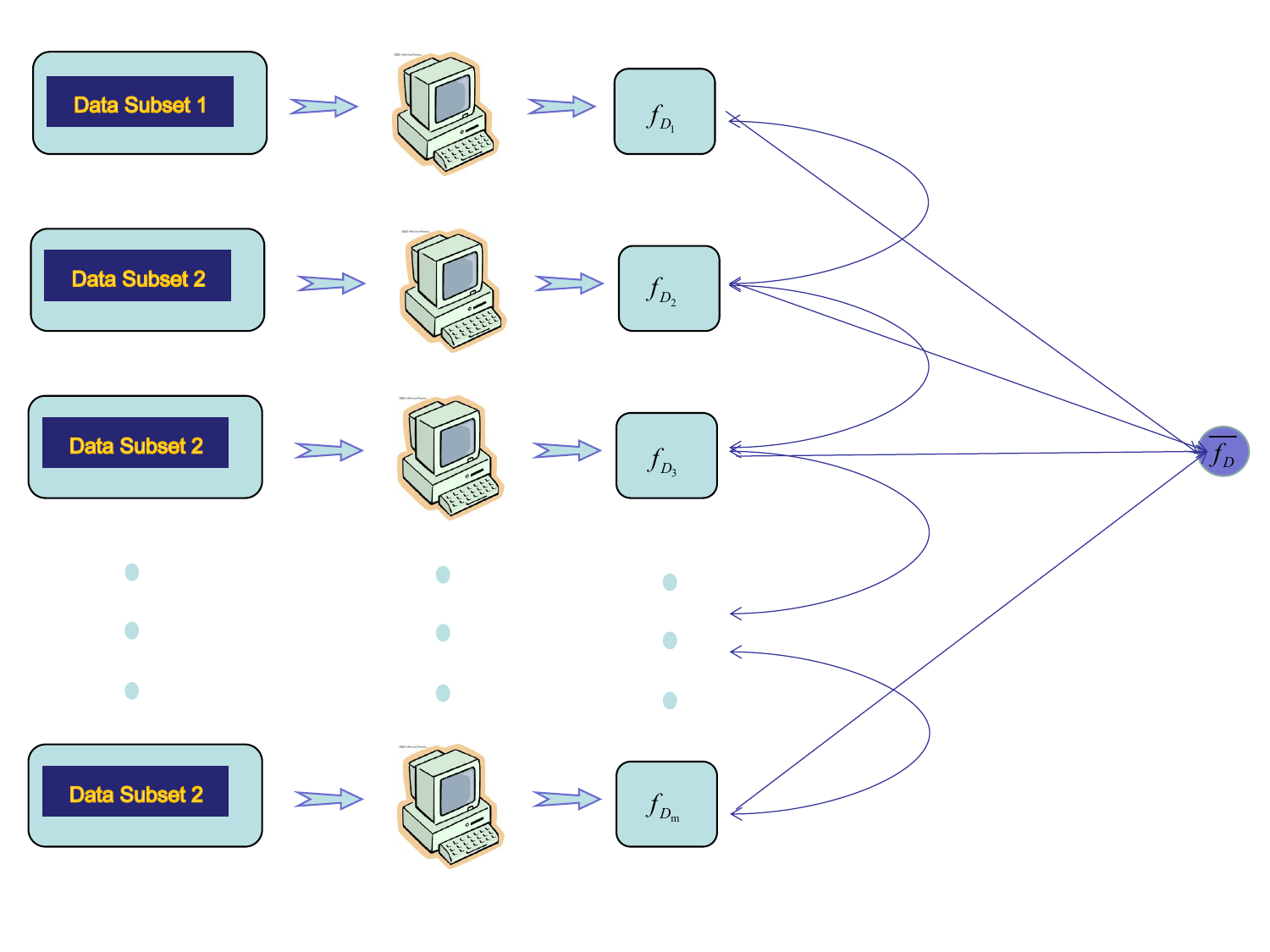}	
\caption{Training flow of distributed learning}
\end{center}
\end{figure}

Since  communications  among different local machines bring additional risk of  privacy disclosure \cite{Wang2018},  we consider distributed  learning without  communications among local machines, where only local processing and synthesization are adopted in the   distributed system \cite{Zhang2015,Huang2019}.
If the data are i.i.d. drawn, it was proved  in \cite{Zhang2015,Chang2017a,Lin2018CA,Mucke2018}  that such a distributed learning scheme performs the same as its batch counterpart that  runs  learning algorithms with whole data   stored on a large enough machine, provided  the number of local machines is not too large. The problem is, however, that such nice results depend  heavily on the i.i.d. nature of data and cannot be extended to non-i.i.d. sequences directly. In fact, it remains open to demonstrate the feasibility of distributed learning  in tackling distributively stored non-i.i.d. data, especially for the theoretical side. The aim of this paper is to present theoretical verifications to guarantee the feasibility and efficiency of distributed learning for non-i.i.d. data.

\subsection{Challenges for analysis}
In this paper, we take the widely used distributed kernel ridge regression (DKRR) \cite{Zhang2015,Lin2017,Chang2017} for example to illustrate the feasibility of distributed learning  to tackle distributively stored strong mixing sequences (or $\alpha$-mixing sequences) \cite{Doukhan1994,Modha1996}. Our results  can be easily  extended to other kernel-based learning algorithms such as the distributed kernel-based gradient descent algorithm  \cite{Lin2018CA} and distributed kernel-based spectral algorithms \cite{Guo2017,Mucke2018,Linj2018}.

To study the learning performance of DKRR, we cast our analysis in the framework of statistical learning theory \cite{Cucker2007} to derive its optimal learning rates\footnote{In this paper, we use the terminology ``learning rate"  to refer the relation between the prediction accuracy and number of  samples, which is popular in learning theory \cite{Cucker2007,Smale2007}, and is different from the terminology in optimization to mean the step-size for  iteration algorithms.}. Different from i.i.d. samples,   the dependence nature of non-i.i.d. data plausibly  reduces the number of  effective data \cite{Yu1994,Modha1996}. Therefore, the existing   analysis \cite{Yu1994,Modha1996} showed  that learning rates of some specific learning algorithms for non-i.i.d. data are always worse than those for i.i.d. data. In particular, it is still unclear   whether the learning rates   established for kernel ridge regression (KRR) \cite{Xu2008,Steinwart2009}  are optimal. In this way, there exist two challenges to analyze the learning performance of DKRR for non-i.i.d. data. The  one is to deduce  optimal learning rates of KRR to provide a baseline and the other is to derive the optimal learning rates of DKRR.

The existing approach to derive  learning rates of KRR for strong mixing data is the empirical process technique \cite{Dehling2002} equipped with a Bernstein inequality established  in \cite{Modha1996}. Due to the dependence nature, even for geometrical $\alpha$-mixing sequences with large $\gamma_0$ (see (\ref{def-Galpha}) below for the detailed definition), there are only $N^\frac{\gamma_0}{\gamma_0+1}$ effective samples involved in the Bernstein inequality \cite{Modha1996}, making the learning rates of KRR derived via the empirical process technique be much  worse  than those  for i.i.d. samples, where $N$ denotes the size of samples. To establish optimal learning rates of KRR for $\alpha$-mixing sequences, it  is necessary to develop   novel analysis techniques rather than the classical empirical process technique.

Concerning DKRR, the recently developed integral operator approaches \cite{Lin2017,Chang2017,Guo2017,Mucke2018} for i.i.d. samples showed that the generalization error can be divided into approximation error, sample error and distributed error. The approximation error,
independent of the sample, describes the approximation capability of the hypothesis space. The sample error, which measures the quality of samples,  reflects the
power of   (weighted) averaging in  distributed learning.
The distributed error,    describing the  limitation of  distributed
learning,  presents a  restriction on the number of local machines to guarantee the
optimal learning rate. The main tool for this error decomposition is the
relation
$
   E[\langle \xi_i,\xi_j\rangle]=\langle E[\xi_i],E[\xi_j]\rangle,
$
where $\{\xi_i\}_{i=1}^\infty$ are i.i.d. drawn.    However, this perfect covariance equality  does not hold for non-i.i.d. sequences, making the classical error decomposition strategy for DKRR \cite{Lin2017,Chang2017,Guo2017,Mucke2018} be unavailable. A novel error decomposition scheme  for   non.i.i.d. sequences  is thus   required.

\subsection{Related Work}

KRR is one of the most popular learning algorithms in learning theory \cite{Cucker2007,Steinwart2008}. Its theoretical behaviors for i.i.d. samples  were sufficiently explored   \cite{Caponnetto2007,SteinwartHS,Lin2017} by showing that KRR can achieve   optimal learning rates. 
Studying learning performance of KRR for strong mixing sequences is also a classical and long-standing research topic  and can date back to 1996, when \cite{Modha1996} derived a Bernstein-type concentration inequality for geometrical  $\alpha$-mixing sequences.
This  nice result  was  adopted  in \cite{Xu2008,Steinwart2009,Hang2014} to derive learning rates for KRR. The problem is, however, that the concentration inequality for $\alpha$-mixing sequences is somewhat worse than that for i.i.d. samples, since the dependence of data is doomed to reduce the number of effective samples. As a result, all learning rates derived from such a Bernstein-type concentration inequality   are worse than the optimal learning rates, no matter how fast the $\alpha$-mixing coefficients decay.
%
Noticing this dilemma, \cite{Sun2010} is the first work, to the best of our knowledge, to study the learning rate without using  the Bernstein-type concentration inequality. Instead, \cite{Sun2010} utilized the integral operator approach in \cite{Smale2007} to derive learning rates of KRR for strong mixing sequences.  However, their results did not take the regularity of kernel into account, making the derived learning rates be sub-optimal. In summary,  the challenge for KRR  proposed in the previous subsection is not settled in these existing  literature.

Distributed learning  based on kernel methods as described in Figure 1 is a powerful approach to tackle  distributively stored massive data. Learning rates analysis of DKRR has been also conducted in \cite{Zhang2015,Lin2017,Mucke2018,Linj2018,Lin2020} for i.i.d. samples. In particular, optimal learning rates for DKRR have been established  \cite{Zhang2015,Lin2017} in the sense that the distributed learning based on a divide-and-conquer scheme performs almost the same as its batch counterpart   provided that the number of local machines is not so large. The main tools for the analysis in \cite{Zhang2015,Lin2017,Mucke2018,Linj2018,Lin2020}  are   operator  concentration inequalities \cite{Smale2007,Caponnetto2007} for i.i.d. samples  and an exclusive  error decomposition strategy based on the independence of samples. Since we are interested in non-i.i.d. data, both tools are unavailable and therefore the challenge for DKRR stated in the previous   remains open.

Different from these interesting  results, we aim at developing a novel   framework to analyze the learning performance of KRR and DKRR for $\alpha$-mixing sequences. We devote to deriving   optimal learning rates for  KRR and DKRR under mild conditions for $\alpha$-mixing sequences.
It should be mentioned  that ``distributed learning'' in this paper is indeed a  parallel computation approach that simply synthesizes   results of all local machines  on a central global machine, and is different from the widely used distributed learning scheme \cite{Boyd2005} that trades all nodes equally without a central agent. We declare that the reason to adopt ``distributed learning'' in this paper  is to highlight that our algorithm is designed for distributively stored data that cannot be shared with each other, and is different from the classical parallel computation that is only developed for   acceleration.
Such a nomination system   has been utilized  in large literature via naming the classical distributed learning scheme \cite{Boyd2005}  as ``decentralized distributed learning'' \cite{Olfati-Saber2007,Nedic2009,Koppel2018} and the parallel computation  with  a central agent for distributively stored data as
``centralized distributed learning'' \cite{Zhang2013,Duchi2014,Lee2017,Jordan2018,Volgushev2019,Linj2018,Chen2019,Wang2020}.

\subsection{Contributions  }
Our  purposes in this paper is to derive optimal learning rates of DKRR for strong mixing sequences to circumvent  the  above challenges.
Our novelty can be attributed  into the following three aspects.

$\bullet$ {\it Methodology:} As mentioned in the previous subsection, empirical process techniques associated with Bernstein-type concentration inequalities \cite{Yu1994,Modha1996} are classical tools to derive learning rates  of KRR for strong mixing sequences. In this paper, we  formulate  a novel integral operator theory framework   to analyze the learning performance of KRR. By the help of a covariance inequality for    Banach-valued strong mixing sequences \cite{Dehling1982},   we also  develop   a  novel error  decomposition strategy for DKRR in the proposed framework. The main advantage  of the integral operator theory  framework is its capability in deriving optimal learning rates of KRR and DKRR.

$\bullet$ {\it Theory:}  We deduce a sufficient condition for the mixing property  to guarantee almost  optimal learning rates for KRR. In particular, we prove that under some standard assumptions \cite{Caponnetto2007,Lin2017},  KRR for  strong mixing sequences  can achieve almost optimal learning rates of KRR for i.i.d. samples \cite{Caponnetto2007,Lin2017}. This is the first result, to the best of our knowledge, to show  optimal learning rates for learning $\alpha$-mixing sequences.
  By the aid of the novel error decomposition strategy and the Banach-valued variance inequality in \cite{Dehling1982} (see Lemma \ref{Lemma:tool1} below), we  also deduce
optimal learning rates for DKRR under some restrictions on the number of local machines. Our results show that   DKRR is a feasible strategy to solve the distributive storage and dependence issues of massive data challenges.

$\bullet$ {\it Experiments:} We conduct both toy simulations and real data experiments to verify our theoretical assertions via exhibiting the excellent learning performance of   DKRR in learning distributively stored time series. We find that DKRR can essentially improve the learning performance for running KRR with data subset  stored on arbitrary local machine. Furthermore, if the number of machines is not so large, DKRR performs similarly as running KRR on the  whole data. These findings  numerically illustrate the power of  distributed learning in tackling distributively stored non i.i.d. data.

 %
%
%
%
%
%

The rest of the paper is organized as follows. In the next section, we  present distributed kernel ridge regression for distributively stored  strong mixing sequences. In Section \ref{Sec:Main-Results}, we give main results of the paper,  including optimal learning rates for KRR and DKRR. 
In Section \ref{Sec.Experiments}, we conduct numerical experiments   to verify our theoretical assertions.
In Section \ref{Sec:tools}, we introduce the main tools for analysis. We prove  our results in the last section.

\section{DKRR for Strong Mixing Sequences}\label{Sec:Strong-Mixing}

In this section, we introduce  strong mixing sequences to quantify the dependence of data and   present DKRR for distributively stored strong mixing sequences.

\subsection{Strong mixing sequences}
It is impossible to establish satisfactory  generalization error bounds for learning algorithms without imposing any restrictions on the dependence.  An extreme case is that samples are copies  of a single data, making the  number of effective samples   be 1.    Strong mixing condition \cite{Rosenblatt1956} is a widely used restriction to quantitatively describe the dependence among different samples and  is much weaker than the  $\beta$-mixing condition \cite{Yu1994} and $\phi$-mixing condition \cite{Sun2010}.

For two $\sigma$-fields $\mathcal J$ and $\mathcal K$, define the $\alpha$-mixing (or strong mixing) coefficient as
\begin{equation}\label{alpha}
     \alpha(\mathcal J,\mathcal K):=\sup_{A\in\mathcal J, B\in\mathcal K}|P(A\cap B)-P(A)P(B)|.
\end{equation}
Let $\{z_i\}_{i=1}^\infty$ be a set of random sequences. Denote by $\mathcal M_{i,j}$ the $\sigma$-filed generated by random variables $z_{i:j}:=(z_i,z_{i+1},\dots,z_j)$.
The strong  mixing condition is defined  as follows.
\begin{definition}\label{Def:alpha}
A set of random sequence $\{z_i\}_{i=1}^\infty$ is said to satisfy a strong  mixing condition (or $\alpha$-mixing condition) if
\begin{equation}\label{def-alpha}
       \alpha_j:=\sup_{k\geq 1}\alpha(\mathcal M_{1,k},\mathcal M_{k+j,\infty})\rightarrow0, \qquad\mbox{as}\ j\rightarrow \infty.
\end{equation}
 If  there are some constants $b_0>0,c_0^*\geq0,\gamma_0 >0$ such that
\begin{equation}\label{def-Galpha}
           \alpha_j\leq c_0^*\exp(-b_0j^{\gamma_0}),\qquad \forall \ j\geq 1,
\end{equation}
then   $\{z_i\}_{i=1}^\infty$ is said    to be geometrical $\alpha$-mixing.
 If  there are some constants $c^*_1>0, \gamma_1 >0$ such that
\begin{equation}\label{def-Galpha11}
           \alpha_j\leq c_1^*j^{-\gamma_1},\qquad \forall \ j\geq 1,
\end{equation}
then   $\{z_i\}_{i=1}^\infty$ is said    to be algebraic $\alpha$-mixing.
\end{definition}

It is easy to derive from Definition \ref{Def:alpha} that   i.i.d. random sequences imply    $\alpha_j=0$ for all $j\geq 1$.
Thus,   strong mixing is a reasonable extension of the classical i.i.d. sampling. As we focus on distributed learning,
the following two lemmas that  present mixing properties for distributively stored strong mixing sequences are needed.

\begin{lemma}\label{Lemma:mixing-independent-1}
Let $m\in\mathbb N$ and $Z_k=\{z_{k,i}\}_{i=1}^{\infty}$ be a sequence of strong mixing sequences for  each $k=1,\dots,m$. Define  $Z:=\{z_i\}_{i=1}^\infty$ with $z_i=h_i(z_{1,i},\dots,z_{m,i})$, where $h_i:\mathbb R\times\cdots\times\mathbb R\rightarrow\mathbb R$ is a Borel function. If $Z_k, k=1,\dots,m$ are independent of each other, then for each $j\geq 1$ there holds
$$
    \alpha_j(Z)\leq \sum_{k=1}^m\alpha_j(Z_k),
$$
where $\alpha_j(Z)$ denotes the $\alpha$-mixing coefficient $\alpha_j$ with respect to the random sequences $Z$.
\end{lemma}

Lemma \ref{Lemma:mixing-independent-1} can be found in \cite[Theorem 5.2]{Bradley2005}, which is an extension of results     in \cite[p.73]{Pinsker1964} for  $\beta$-mixing sequences. A consequence  of Lemma \ref{Lemma:mixing-independent-1} with $m=1$ is that acting arbitrary Borel function on $\alpha$-mixing   maintains the  mixing coefficients.

\begin{lemma}\label{Lemma:mixing-independent}
Let $m\in\mathbb N$ and $Z_k=\{z_{k,i}\}_{i=1}^{\infty}$ be a sequence of strong mixing sequences for  each $k=1,\dots,m$. Define   $Z^*=\{z^*_i\}_{i=1}^\infty$ with $z^*_i=(z_{1,i},\dots, z_{m,i})$. If $Z_k, k=1,\dots,m$ are independent of each other,
then for each $j\geq 1$, there holds
$$
       \alpha_j(Z^*)\leq \sum_{k=1}^m\alpha_j(Z_k).
$$
\end{lemma}

Lemma \ref{Lemma:mixing-independent} can be derived from \cite[Lemma 8]{Bradley1981} (see also \cite[Theorem 5.1]{Bradley2005}) directly. It presents solid  guarantee for   $\alpha$-mixing condition  of distributively stored dependent data.
In the following, we provide several examples in time series and dynamical systems  to generate   strong mixing sequences.

{\it Example 1 (Nonparametric ARX(p,q) model):}  Suppose that $\{z_t\}_{t=1}^{\infty}$ is generated according to
$$
        z_t=f_0(z_{t-1},\dots,z_{t-p}, z'_{t},\dots,z'_{t-q+1})+e_t,
$$
where $\{z'_t\}$ and $\{e_t\}$ are independent with $E[e_t]=0$. Then it can be found in \cite[p.102]{Doukhan1994} (see also \cite[Proposition 1]{Chen1998}) that under certain conditions on $\{z'_t\}$ and $\{e_t\}$ and some boundedness assumption of $f_0$, $\{z_t\}_{t=1}^{\infty}$ is strong mixing.

{\it Example 2:
  (ARMA process)} Suppose that the time series  $\{z_t\}_{t=1}^{\infty}$ satisfies the ARMA equation
$$
      \sum_{i=1}^pB_iz_{t-i}=\sum_{k=0}^qA_k\varepsilon_{t-k},
$$
where $B_i$, $A_k$ are real matrices and $\{\varepsilon_t\}$ is a sequence of  i.i.d. random vectors. Then, it can be found in \cite{Mokkadem1988} that under some restrictions on $\{\varepsilon_t\}$, $\{z_t\}_{t=1}^{\infty}$ is strong mixing.

{\it Example 3 (Dynamic Tobit process):} Suppose that $\{z_t\}_{t=1}^{\infty}$ satisfies
 $$
     z_t=\max\{0,\xi_0z_t'+\eta_0z_{t-1}+\omega_0+\varepsilon_t\},
$$
where $\xi_0,\eta_0,\omega_0$ are fixed parameters, $z_t'$ is  an exogenous regressor and $\varepsilon_t$ is   a disturbance term. It was proved in \cite{Hahn2010} that under certain conditions on $\xi_0,\eta_0,\omega_0,x_t', \varepsilon_t$, $\{z_t\}_{t=1}^{\infty}$ is strong mixing.

Besides the proposed examples, there are numerous strong mixing sequences including AR(p) process, MA(p) process, ARIMA process, nonlinear ARCH process, GARCH process,  Harris Chains and linear time-invariant  dynamical systems. We refer the readers to \cite{Doukhan1994} for more examples of strong mixing sequences.

\subsection{DKRR for strong mixing sequences}
Let $D_k=\{z_i\}_{i=1}^{|D_k|}=\{(x_{i,k},y_{i,k})\}_{i=1}^{|D_k|}$ be a realization of some stochastic process $(x_t,y_t)$, where $|D_k|$ denotes the cardinality of $D_k$.  Assume that $D_k$ is stored on   the $k$-th data subset and $D_k\cap D_{k'}=\varnothing$ for $k\neq k'$.
Strong mixing condition describes the relation between past random variables  and future random variables and therefore poses   strict restrictions on the order of samples, which brings additional difficulty to design distributed learning algorithms.
In the following, we introduce two popular stages in which  distributed learning for strong mixing sequences is highly desired.

{\it Stage 1 (Sequential case):}
Let $D$ be a  strong mixing sequence. The size of $D$ is so large that it must be sequentially stored on $m$ local machines $D_1,\dots,D_m$ with a suitable order. In such a way, for each $k=1,\dots,m$, $D_k$ is   strong mixing and distributed learning  for strong mixing sequences without data communications is needed.

{\it Stage 2 (Parallel case):}  Let $\{D_k\}_{k=1}^m$ with $D_k=\{(x_{i,k},y_{i,k})\}_{i=1}^{|D_k|}$ be
$m$ strong mixing sequences  that belong to $m$ organizations (companies). Each organization is desired to access others' data to improve its quality of decision-making, which is practically impossible due to the data privacy.  It is thus highly desired to develop a distributed learning algorithm  without data sharing that behaves similarly as its batch counterpart with all data $D=\bigcup_{k=1}^m D_k$ being accessed.


In summary, we are interested in developing a distributed learning system via a global machine and $m$ local machines to tackle distributively stored strong mixing sequences. For the $k$-th local machine with $1\leq k\leq m$, there is a data subset $D_k$ that cannot be   communicated to other local machines. The aim is to yield a global estimator of high quality. We adopt the classical DKRR \cite{Zhang2015,Lin2017,Chang2017} for this purpose.

 Let $K(\cdot,\cdot)$ be a Mercer kernel \cite{Cucker2007} on a compact metric   space ${\mathcal X}$ and $({\mathcal H}_K, \|\cdot\|_K)$ be
the corresponding reproduced kernel Hilbert space (RKHS).  Then, for any $f\in\mathcal H_K$, the reproducing property yields
\begin{equation}\label{norm-c-RKHS}
     \|f\|_\infty=\|\langle f,K_x\rangle_K\|_\infty\leq\kappa\|f\|_K,
\end{equation}
where $K_x:=K(x,\cdot)$.

DKRR is   defined \cite{Zhang2015,Lin2017} with a
regularization parameter $\lambda>0$ by
\begin{equation}\label{DKRR}
     \overline{f}_{D,\lambda}= \sum_{k=1}^m \frac{|D_k|}{|D|} f_{D_k,
       \lambda},
\end{equation}
where
\begin{equation}\label{KRR}
    f_{D_k,\lambda} =\arg\min_{f\in \mathcal{H}_{K}}
    \left\{\frac{1}{|D_k|}\sum_{(x, y)\in
    D_k}(f(x)-y)^2+\lambda\|f\|^2_{K}\right\}.
\end{equation}
If the samples in $\{D_k\}_{k=1}^m$ are i.i.d. drawn  and $m$ is not so large, DKRR is proved to perform similarly as implementing KRR
with whole data $D=\bigcup_{k=1}^mD_k$, i.e.,
\begin{equation}\label{KRR1}
    f_{D,\lambda} =\arg\min_{f\in \mathcal{H}_{K}}
    \left\{\frac{1}{|D|}\sum_{(x, y)\in
    D}(f(x)-y)^2+\lambda\|f\|^2_{K}\right\}
\end{equation}
via showing the same optimal learning rates of $\overline{f}_{D,\lambda}$ and $f_{D,\lambda}$ \cite{Zhang2015,Lin2017,Chang2017,Mucke2018}.
However,  the learning performance of DKRR  remains open when the i.i.d. assumption is removed. Indeed, it is even unknown whether KRR \eqref{KRR1}, a reference of DKRR \eqref{DKRR},  can reach the optimal learning rates for  strong mixing sequences.

%
%
%
%
%
%
%
%
%

\section{Main Results}\label{Sec:Main-Results}
In this section, we present our main results on analyzing   learning performances of DKRR for distributively stored  strong mixing data. As a byproduct, we also provide a sufficient condition to guarantee optimal learning rates for KRR.

\subsection{Setup and assumptions}
Let $k=1,\dots,m$. On each local machine, $D_k=\{z_{k,i}\}=\{(x_{k,i},y_{k,i})\}$ with $x_{k,i}\in\mathcal X$ and $y_{k,i}\in\mathcal Y\subseteq\mathbb R$, $k=1,\dots,m$, is a realization of some stochastic process $(x_t,y_t)$.
We at first introduce the well known  strict  stationarity of random sequences.
\begin{definition}\label{Def:Stationarity}
A random sequence $\{z_{i}\}_{i=1}^\infty$ is  strictly  stationary if for all $i$ and all non-negative integers $j$ and $k$, the random vectors $z_{i:i+j}$ and $z_{i+k:i+j+k}$ have the same distribution.
\end{definition}

According to Definition \ref{Def:Stationarity} with $j=0$, strict stationarity shows that the marginal distribution of $z_i$ is independent of $i$.  Furthermore, Definition \ref{Def:Stationarity} with $j=1$ implies that the bivariate distribution of $z_i$ and $z_j$ must be the same as that of $z_{i+k}$ and $z_{j+k}$, from which it follows that
$
   Cov(z_i,z_{i+1})=Cov(z_{j},z_{j+1}),
$
where $Cov(z_i,z_j)$ denotes the covariance of random variables $z_i$ and $z_j$. Though certain restrictions on the distribution is necessary \cite{Adams2010},  the requirement of the same of joint distributions excludes  some well known time series models and random processes \cite[Chap.4]{Wei2006}.
In this paper, we relax the strict stationarity to the invariance of marginal distribution.

{\it Assumption 1:} Denote the marginal distributions of $z_{k,i}$ are $\rho_{k,i}$. Assume that   $\rho_{k,i}$ is independent of $k$ or $i$.

Assumption 1 with $k=1$ is an essential relaxation of the strict stationarity assumption. For example, let $F(z_1,z_2)$ and $G(z_1,z_2)$ be two   bivariate distributions with density
\begin{eqnarray*}
    f(z_1,z_2)&=&\left\{\begin{array}{cc}
    z_1+z_2 & \mbox{if}\ 0\leq z_1,z_2\leq 1,\\
    0 &\mbox{otherwise}
    \end{array}
    \right.\\
    g(z_1,z_2)&=&\left\{\begin{array}{cc}
    (0.5+z_1)(0.5+z_2) & \mbox{if}\ 0\leq z_1,z_2\leq 1,\\
    0 &\mbox{otherwise.}
    \end{array}
    \right.
\end{eqnarray*}
It is obvious that the $F$ and $G$ are different distributions but their marginal distribution are  same, i.e.
\begin{eqnarray*}
   \int_{-\infty}^\infty f(z_1,z_2)dz_2&=& \int_{-\infty}^\infty g(z_1,z_2)dz_2=0.5+z_1,\\
   \int_{-\infty}^\infty f(z_1,z_2)dz_1&=& \int_{-\infty}^\infty g(z_1,z_2)dz_1=0.5+z_2.
\end{eqnarray*}

%
%
%

Under this circumstance, Assumption 1 with $k=1$ (non-distributed version) has been widely adopted in numerous papers to quantify the stability of random processes in \cite{Yu1994,Modha1996,Xu2008,Steinwart2009,Sun2010,Feng2012,Hang2014,Hang2017}. The main advantage of employing Assumption 1 in analysis is its similar analysis setup as the standard learning theory setting for i.i.d. samples \cite{Cucker2007,Steinwart2008}. In fact, due to Assumption 1, denote the marginal distribution of $z_{k,i}$ to be  $\rho:=\rho_X\times \rho(y|x)$. The purpose is then to find an estimator $f_D$ based on $D_k,k=1,\dots,m$ to minimize the expectation  risk
$$
       \mathcal E(f_D)=\int (f_D(x)-y)^2d\rho.
$$
Noting that the well known regression function $f_\rho(x)=\int_{\mathcal Y}yd\rho(y|x)$ minimizes the expectation  risk \cite{Cucker2007},  the learning  task is then to find an estimator $f_D$ to minimize
$$
    \mathcal E(f_D)-\mathcal E(f_\rho)=\|f_D-f_\rho\|_\rho^2,
$$
where $\|\cdot\|_\rho$ denotes the norm of the Hilbert space $L^2_{\rho_X}$.

It should be mentioned that Assumption 1 can be also relaxed to the condition that   regression functions  $f_{\rho_{k,i}}$ are independent of $k,i$ (with a slight change of analysis)  to include some well known time series such as the random walking. Since our purpose in this paper is to show the feasibility of distributed learning, we conduct our analysis in a standard regression framework in \cite{Yu1994,Modha1996,Xu2008,Steinwart2009,Sun2010,Feng2012,Hang2014,Hang2017}, which requires the identical distribution assumption. To derive the learning rates, we also need four popular types of assumptions concerning the mixing property of samples, boundedness of  outputs, capacity of  $\mathcal H_K$ and regularity of $f_\rho$,   respectively.

{\it Assumption 2:} For each $k=1,\dots,m$, assume   that  $D_k$ is a strong mixing sequence with  mixing coefficient $\alpha_j$. Assume further that  there is a suitable arrangement of $D_1,\dots,D_m$ such that $D=\bigcup_{k=1}^mD_k$ is also a strong mixing sequence with $\alpha$-mixing coefficient $\alpha_j$.

The  strong mixing assumption  is widely used  to quantify the dependence of samples. It has been adopted in  \cite{Modha1996,Xu2008,Sun2010,Feng2012,Hang2014} to derive learning rates for KRR and is looser than the widely used $\beta$-mixing condition in \cite{Yu1994,Chen1998,Alquier2013}. If   data are distributively   stored in a sequential manner (Stage 1 in Section II. B), then Assumption 2 naturally holds. If data are distributively   stored in a parallel manner (Stage 2 in Section II. B),
 due to Lemma \ref{Lemma:mixing-independent-1} and Lemma \ref{Lemma:mixing-independent}, with suitable scaling of the $\alpha$-mixing coefficient,
the $\alpha$-mixing property of $D$ can be guaranteed by the $\alpha$-mixing property of each data subset and the independence between different data subsets.
Assumptions 1 and 2 are necessary for distributed learning in the sense that all non-i.i.d. data subset distributively stored on different local machines are used to predict the unknown but definite  regression function $f_\rho$.

%

{\it Assumption 3:} There exists a positive constant $M$ such that $|y|\leq M$ almost surely.

Since we are always concerned with learning problems with   finitely many samples, it is easy to derive an upper bound of the output. From Assumption 3, it follows  that $|f_\rho(x)|\leq M$ almost surely for any $x\in\mathcal X$.
The Mercer kernel $K: {\mathcal X}\times {\mathcal X}
\rightarrow \mathcal R$ defines an integral operator $L_K$ on
${\mathcal H}_K$ (or $L_{\rho_X}^2$) by
$$
         L_Kf =\int_{\mathcal X} K_x f(x)d\rho_X, \qquad f\in {\mathcal
          H}_K \quad (\mbox{or}\ f\in L_{\rho_X}^2).
$$
Our third assumption is on the capacity of $\mathcal H_K$ measured by the
effective dimension \cite{Guo2017,Lin2017},
$$
        \mathcal{N}(\lambda)={\rm Tr}((\lambda I+L_K)^{-1}L_K),  \qquad \lambda>0,
$$
where $\mbox{Tr}(A)$ denotes the trace of the trace-class operator $A$.

{\it Assumption 4:
 There exists some $s\in(0,1]$ such that
\begin{equation}\label{assumption on effect}
      \mathcal N(\lambda)\leq C_0\lambda^{-s},
\end{equation}
where $C_0\geq 1$ is  a constant independent of $\lambda$.
}

 Condition (\ref{assumption on
effect}) with $s=1$ is always satisfied by taking
$C_0=\mbox{Tr}(L_K)\leq\kappa^2$. For $0<s<1$,   it was shown in
\cite{Guo2017} that the assumption is
more general than the eigenvalue decaying assumption in the
literature \cite{Caponnetto2007,SteinwartHS,Zhang2015}. Assumption 4 has been employed in
\cite{Guo2017,Lin2017,Chang2017} to
derive optimal learning rates for kernel-based learning algorithms. Our last assumption is the well known regularity restriction on the regression function.

{\it Assumption 5:
For $r>0$, assume
\begin{equation}\label{regularitycondition}
         f_\rho=L_K^r h_\rho,~~{\rm for~some}  ~ h_\rho\in L_{\rho_X}^2,
\end{equation}
where $L_K^r$ denotes the $r$-th power of $L_K: L_{\rho_X}^2 \to
L_{\rho_X}^2$.
}

Condition (\ref{regularitycondition}) describes the
regularity of $f_\rho$ and has been adopted in   large literature  \cite{Smale2007,Caponnetto2007,Guo2017,Lin2017} to
quantify  learning rates for some algorithms. It should be noted that if (\ref{regularitycondition}) holds for some $r>0$, then $f_\rho=L_K^{r_0} L_K^{r-r_0}h_\rho$ for $0<r_0<r$, implying that the $(r,h_\rho)$ pairs satisfying (\ref{regularitycondition}) are not unique. We highlight that, as shown  in \cite{Smale2007,Caponnetto2007,Guo2017,Lin2017}, Assumption 5 means to select  the largest $r$ satisfying (\ref{regularitycondition}).
 In summary, there are five assumptions on the stationarity, mixing, boundedness of samples, capacity of RKHS and regularity of the regression functions   in our analysis. If data are i.i.d. samples, Assumptions 3-5 are standard in learning theory \cite{Caponnetto2007,Blanchard2016,Chang2017,Lin2017,SteinwartHS}, based on which optimal learning rates for KRR have been derived. Under Assumptions 1-5, which have been adopted in \cite{Steinwart2009,Sun2010} and can be simultaneously  satisfied for Examples 1-3 with some specific parameters \cite{Hahn2010}, we focus on deriving optimal learning rates for DKRR.

\subsection{Optimal learning rates for KRR}

To study the learning performance of DKRR, we should provide a baseline for analysis, where optimal learning rates of KRR are required.
KRR is one of the  most popular learning algorithms in learning theory \cite{Cucker2007,Steinwart2008}.  Based on Bernstein-type concentration inequalities \cite{Yu1994,Modha1996},  generalization error bounds for  non-i.i.d. sequences have   been derived in \cite{Xu2008,Steinwart2009,Hang2014} under some mixing conditions. However, compared with the results in \cite{Caponnetto2007,SteinwartHS,Chang2017,Lin2017,Fang2019} for i.i.d. samples, there is a dilemma in the existing literature that the dependence of samples always degenerates the learning performance of KRR, no matter which mixing condition is imposed. An extreme case is that even for geometrical $\alpha$-mixing sequences satisfying (\ref{def-Galpha}) with very large $\gamma_0$ \cite{Xu2008}, there  is still a gap  between learning rates for  i.i.d. samples and non-i.i.d. sequences.

Our first purpose in this section is to fill the aforementioned  gap by presenting a sufficient condition for the $\alpha$-mixing coefficient to guarantee the almost optimal learning rates of KRR.
The following theorem presents   learning rates analysis of KRR under the $\mathcal H_K$ norm.

 \begin{theorem}\label{Theorem:KRR-err-dec-HK-1}
  Under Assumption 1-Assumption 5 with $m=1$, $0<s\leq 1$ and $1/2\leq r\leq 1$, if $\lambda=|D|^{-1/(2r+s)}$, then for arbitrary $\delta>0$, there holds
\begin{eqnarray}\label{KRR-err-HK-1}
    E[\|f_{D,\lambda}-f_\rho\|_K]
    &\leq& C \left(1+|D|^{-\frac12+\frac{(s+1)(\delta+1)}{(2r+s)(\delta+2)}}
    \sum_{\ell=1}^{|D|}(\alpha_\ell)^{\frac\delta{4+2\delta}}\right) \\
    &\times&
    \left( |D|^{-\frac{r-1/2}{2r+s}}+|D|^{-\frac12+\frac{2(s+1)(\delta+1)+\delta+2 }{(2r+s)(2\delta+4)}}
    \sum_{\ell=1}^{|D|}(\alpha_\ell)^{\frac\delta{4+2\delta}}\right) ,\nonumber
\end{eqnarray}
where $C$ is a constant independent of $|D|$.
\end{theorem}

It should be mentioned that the main advantage in our analysis in \eqref{KRR-err-HK-1} is an additional power factor $\delta$, which focuses on alleviating the negative effects of dependence. The reason for such an advantage is that we adopt a Banach-valued covariance inequality (see Lemma \ref{Lemma:tool1} in Section \ref{Sec:tools}) that involves a power factor  to  balance the effects of $\alpha$-mixing coefficient and boundedness of the random variable. In particular, if $\alpha_j=0$, i.e. for i.i.d. samples, Theorem \ref{Theorem:KRR-err-dec-HK-1} implies
$$
  E[\|f_{D,\lambda}-f_\rho\|_K]\leq C|D|^{-\frac{r-1/2}{2r+s}},
$$
which coincides with the optimal learning rates established  \cite{Caponnetto2007,Lin2017,Chang2017}.
If $\alpha_j$ satisfies (\ref{def-Galpha11}), for arbitrary $\varepsilon>0$,
 we can set
$
    \delta=\frac{4(2r+s)\varepsilon}{2+s-2(2r+s)\varepsilon}
$
to derive the following error estimate.

\begin{corollary}\label{Corollary:KRR-err-dec-HK-1}
Let $\varepsilon>0$ be an arbitrary small integer. Under Assumption 1-Assumption 5 with $m=1$, $\alpha_j$  satisfying (\ref{def-Galpha11}), $0<s\leq 1$ and $1/2\leq r\leq 1$,
 if $\lambda=|D|^{-1/(2r+s)}$ and $\gamma_1$ in (\ref{def-Galpha11}) satisfies
\begin{equation}\label{suff-cond-1}
    \gamma_1>2+\max\left\{\frac{s+2}{(2r+s)\varepsilon},\frac{4(s+2-2r)}{4r-2}\right\},
\end{equation}
 then
\begin{eqnarray}\label{Upper-KRR-HK}
    E[\|f_{D,\lambda}-f_\rho\|_K]
    &\leq& C' |D|^{-\frac{r-1/2}{2r+s}+\varepsilon},
\end{eqnarray}
where $C'$ is a constant independent of $|D|$.
\end{corollary}

Corollary \ref{Corollary:KRR-err-dec-HK-1} showed that for $\alpha$-mixing sequences with $\alpha_j$ satisfying (\ref{def-Galpha11}) and $\gamma_1$ satisfying
(\ref{suff-cond-1}), KRR  achieves the almost optimal learning rates, filling the gap between the classical error analysis of KRR for dependent samples \cite{Xu2008,Steinwart2009,Hang2014}  and i.i.d. samples \cite{Caponnetto2007,SteinwartHS,Chang2017,Lin2017}.
Condition (\ref{suff-cond-1}) is mild  for strong mixing sequences \cite{Doukhan1994} in the sense that Corollary \ref{Corollary:KRR-err-dec-HK-1} holds for any geometrical $\alpha$-mixing sequences with extremely small $\varepsilon>0$.

If the generalization error is measured by the $\mathcal H_K$ norm, then Theorem \ref{Theorem:KRR-err-dec-HK-1} and Corollary \ref{Corollary:KRR-err-dec-HK-1} show  that under certain restrictions on the mixing condition, the dependence nature does not essentially reduce the effective number in the expectation   sense. In the following theorem, we
prove that similar results also hold for   { KRR}   under the $L_{\rho_X}^2$ norm. Furthermore, we derive a sufficient condition on the decay of mixing coefficient to guarantee optimal learning rates of KRR.

 \begin{theorem}\label{Theorem:KRR-err-dec-rho-1}
Under Assumption 1-Assumption 5 with $m=1$, $0<s\leq 1$, $1/2\leq r\leq 1$ and $2r+s\geq2$, if $\lambda=|D|^{-1/(2r+s)}$, then for arbitrary $\delta>0$ there holds
\begin{eqnarray}\label{KRR-err-rho-1}
   &&E[\|f_{D,\lambda}-f_\rho\|_\rho] \\
   &\leq &
   \hat{C}|D|^{-\frac{r}{2r+s}} \left(1+\sum_{\ell=1}^{|D|}\sqrt{\alpha_\ell}\right)\left( 1 +|D|^{-\frac12+\frac{2s+\delta+2\delta r+4r}{(2r+s)(2\delta+4)}}\sum_{\ell=1}^{|D|}(\alpha_\ell)^{\frac\delta{4+2\delta}}\right), \nonumber
\end{eqnarray}
where $\hat{C}$ is a constant independent of $|D|$.
\end{theorem}

We then derive almost optimal learning rates for KRR with dependent samples based on Theorem \ref{Theorem:KRR-err-dec-rho-1}. Denote by $\mathcal M(m,c_1^*,\gamma_1,M,C_0,s,r)$ the set of all distributions such that Assumption 1-Assumption 5 hold   with     $\alpha_j$ satisfying (\ref{def-Galpha11}), $M>0$, $s\in(0,1]$, $C_0\geq 1$ and $r>0$.
We enter into a competition over $ \Psi_D$,  which denotes  the class of all functions derived from the data set   $D$  and define
\begin{eqnarray*}
          e(m,c_1^*,\gamma_1,M,C_0,s,r)
           := \sup_{\rho\in \mathcal M(m,c_1^*,\gamma_1,M,C_0,s,r)}\inf_{f_D\in \Psi_D} E(\|f_\rho-f_{D}\|_\rho).
\end{eqnarray*}
For an arbitrarily small $\varepsilon>0$,
setting $\delta=\frac{8r\varepsilon+4s\varepsilon}{1-s-4r\varepsilon-2s\varepsilon}$ in Theorem \ref{Theorem:KRR-err-dec-rho-1}, we obtain the following corollary directly.

\begin{corollary}\label{Corollary:KRR-err-dec-rho-1}
Let $\varepsilon>0$,  $0<s\leq 1$, $1/2\leq r\leq 1$,   $2r+s\geq 2$ and
\begin{equation}\label{suff-cond-2}
    \gamma_1>2+\frac{4-4s}{8r\varepsilon+4s\varepsilon}.
\end{equation}
If  $\lambda=|D|^{-1/(2r+s)}$, then
\begin{eqnarray}\label{Upper-KRR-rho}
 &&\hat{C}''|D|^{-\frac{r}{2r+s}} \leq  e(1,c_1^*,\gamma_1,M,C_0,s,r)\nonumber\\
 &\leq&
 \sup_{\rho\in \mathcal M(1,c_1^*,\gamma_1,M,C_0,s,r)}  E[\|f_{D,\lambda}-f_\rho\|_\rho]
    \leq
   \hat{C}'|D|^{-\frac{r}{2r+s}+\varepsilon}
\end{eqnarray}
where $\hat{C}'$ and $\hat{C}''$ are constants independent of $|D|$.
In particular, if $s=1$, then for any $1/2\leq r\leq 1$, $\gamma_1>2$ implies
\begin{eqnarray*}
 &&\hat{C}''|D|^{-\frac{r}{2r+1}} \leq  e(1,c_1^*,\gamma_1,M,C_0,1,r)\nonumber\\
 &\leq&
 \sup_{\rho\in \mathcal M(1,c_1^*,\gamma_1,M,C_0,1,r)}  E[\|f_{D,\lambda}-f_\rho\|_\rho]
    \leq
   \hat{C}'|D|^{-\frac{r}{2r+1}}.
\end{eqnarray*}
\end{corollary}

Corollary \ref{Corollary:KRR-err-dec-rho-1}  provides a sufficient condition for the mixing property   to guarantee the optimal learning rate of KRR. The rate in
\eqref{Upper-KRR-rho} shows that under (\ref{def-Galpha11}) and (\ref{suff-cond-2}), strong mixing sequences behave similarly as i.i.d. samples for KRR.  Compared with the i.i.d. case further, additional restrictions on the regularity of $f_\rho$ and capacity of $\mathcal H_K$, i.e. $2r+s\geq 2$ are imposed in our results. This is a technical assumption and is much stricter than the widely used condition $1/2\leq r\leq 1$ and $0<s\leq 1$ in \cite{Caponnetto2007,SteinwartHS,Lin2017}.
We believe that it can be relaxed to $1/2\leq r\leq 1$, just as Theorem \ref{Theorem:KRR-err-dec-HK-1} did for the $\mathcal H_K$ norm. We leave it as our future studies. Since $s=1$ and $r\geq 1/2$ implies $2r+s\geq 2$, a special case for Corollary \ref{Corollary:KRR-err-dec-rho-1} is that under  the capacity-independent (without Assumption 4),    KRR for strong mixing sequences can achieve the optimal learning rates of KRR for i.i.d. samples.

\subsection{Optimal learning rates for DKRR}
We are now in a position to study the learning performance of DKRR defined by (\ref{DKRR}). It was shown in \cite{Zhang2015,Lin2017}  that the defined DKRR in (\ref{DKRR}) maintains the learning performance of KRR with whole data, provided the number of local machines is not so large and samples are i.i.d. drawn. Our result shows that under stricter condition on the number of local machines, DKRR also achieves the optimal learning rates for   $\alpha$-mixing sequences  in (\ref{Upper-KRR-rho}), showing that DKRR is a powerful learning scheme for distributively stored dependent data.

\begin{theorem}\label{Theorem:KRR-derr-dec}
Under Assumption 1-Assumption 5 with $0<s\leq 1$, $1/2\leq r\leq 1$ and $2r+s\geq2$, if $\lambda=|D|^{-1/(2r+s)}$, then for arbitrary $\delta>0$ there holds
\begin{eqnarray}\label{KRR-derr-rho-1}
    &&\max\left\{E[\|\overline{f}_{D,\lambda}-f_\rho\|_\rho],|D|^{-\frac{1}{4r+2s}}E[\|\overline{f}_{D,\lambda}-f_\rho\|_K]\right\} \nonumber \\
    &\leq&
     \bar{C}|D|^{-\frac{2r+s-1}{2r+s}}
     \sum_{j=1}^m \left(1+\sum_{\ell=1}^{|D_j|}\sqrt{\alpha_\ell}\right) \nonumber \\
    &\times&
      \left(|D|^\frac{s}{4r+2s}+ |D|^\frac{2s+\delta+1}{(2r+s)(2\delta+4)}\sum_{\ell=1}^{|D_j|}(\alpha_\ell)^{\frac\delta{4+2\delta}}\right) \nonumber \\
     &+&
     \bar{C} \left( |D|^{-\frac{r}{2r+s}}  +|D|^{-\frac12+\frac{2s+\delta+1}{(2r+s)(2\delta+4)}}\sum_{\ell=1}^{|D|}(\alpha_\ell)^{\frac\delta{4+2\delta}}\right).
\end{eqnarray}
where $\bar{C}$ is a constant independent of $|D|$.
\end{theorem}

For an arbitrary $\varepsilon>0$, setting $\delta=\frac{8r\varepsilon+4s\varepsilon}{1-s-4r\varepsilon-2s\varepsilon}$, we have from Theorem \ref{Theorem:KRR-derr-dec} the following corollary.

\begin{corollary}\label{Corollary:KRR-derr-dec}
  Under Assumption 1-Assumption 5 with $0<s\leq 1$, $1/2\leq r\leq 1$ and $2r+s\geq2$, if  $\lambda=|D|^{-1/(2r+s)}$, (\ref{def-Galpha11}) holds with (\ref{suff-cond-2}), $|D_1|=\dots=|D_m|$
and $m$ satisfies
\begin{equation}\label{condition-m}
    m\leq|D|^{\frac{2r+s-2}{4r+2s}},
\end{equation}
 then
\begin{eqnarray}\label{DKRR-op-1}
   E[\|\overline{f}_{D,\lambda}-f_\rho\|_K]\leq \bar{C}'|D|^{-\frac{r-1/2}{2r+s}+\varepsilon},
\end{eqnarray}
and
\begin{eqnarray}\label{DKRR-op-2}
 &&\hat{C}''|D|^{-\frac{r}{2r+s}} \leq  e(m,c_1^*,\gamma_1,M,C_0,s,r)\nonumber\\
 &\leq&
 \sup_{\rho\in \mathcal M(,m,c_1^*,\gamma_1,M,C_0,s,r)}  E[\|\overline{f}_{D,\lambda}-f_\rho\|_\rho]
    \leq
  \bar{C}'|D|^{-\frac{r}{2r+s}+\varepsilon},
\end{eqnarray}
where $\bar{C}'$ is a constant independent of $|D|$.
\end{corollary}

Comparing Corollary \ref{Corollary:KRR-derr-dec} with Corollary \ref{Corollary:KRR-err-dec-HK-1} and Corollary \ref{Corollary:KRR-err-dec-rho-1}, we get that DKRR performs similarly as running KRR on  $\alpha$-mixing sequences  stored on a large enough machine, provided the number of local machines is not so large.  This extends the existing distributed learning theory from i.i.d. samples to strong mixing sequences and extends the applicable range for distributed learning.

Theoretical assessments  for DKRR as described above for i.i.d. samples are  a popular topic  in the realm of learning theory and optimal learning rates  have been established in   \cite{Zhang2015,Lin2017,Guo2017,Mucke2018,Lin2018CA,Linj2018,Lin2020}.
 A general conclusion is that under Assumptions 2-4, if
$ m \leq |D|^{\frac{2r+s-1}{2r+s}}$ and $1/2\leq r\leq 1$, then DKRR  for i.i.d. samples achieves the optimal learning rates like (\ref{Upper-KRR-rho}).
Our Theorem \ref{Theorem:KRR-derr-dec} extends the results from i.i.d. samples to $\alpha$-mixing sequences, the price of which is    a stricter restriction on the number of local machines, i.e. from $ m \leq |D|^{\frac{2r+s-1}{2r+s}}$ to $m\leq|D|^{\frac{2r+s-2}{4r+2s}}$. It should be highlighted that such a degradation is reasonable since the dependence of samples destroys the classical error decomposition for distributed learning  \cite{Guo2017,Chang2017} and requires a novel error analysis technique.  Based on our results, we find that DKRR  is a good candidate to overcome the massive data challenges since it is efficient for distributively stored and dependent data and also provides a feasible way to resolve the privacy issue in the sense that  there are no communications between local machines.

We end this section by discussing the  role of the regularization parameter $\lambda$. As shown in Corollaries 1, 2, 3,   regularization parameters for both KRR and DKRR are   same, implying that  KRR in each local machine over-fits the training samples and the weighted average operator succeeds in reducing the variance. It should be highlighted that in our analysis, the choice of the regularization parameter $\lambda$ in DKRR is independent of  $k$, $k=1,\dots,m$, and is the same as KRR with whole data. This is consistent with \cite{Zhang2015,Lin2017} for i.i.d. samples. However, it is practically difficult to present such a perfect $\lambda$ under the distributed learning framework. Since we only consider the theoretical feasibility  of DKRR in this paper,  we   present an adaptive selection strategy of $\lambda$ without theoretical verifications. We will continue our study on this topic and   present provable parameter selection strategy in a future work. Under the assumptions of this paper, the theoretically optimal regularization parameter satisfies $\lambda\sim |D|^{-1/(2r+s)}$. It is well known \cite[Chap.7]{Gyorfi2002} that it can be realized by using the well known cross-validation approach.
However, in the distributed learning setting, there are only $|D_k|$ data in the $k$-th local machine and we can get a regularization $\lambda_k\sim |D_k|^{-1/(2r+s)}$ via  cross-validation, which leads to sub-optimal learning rates. Theoretically, given $\lambda_k$ derived by cross-validation, we can set $\hat{\lambda}_k= (\lambda_k)^{\log_{|D_k|}|D|}$. Then $\lambda_k\sim |D_k|^{-1/(2r+s)}$ implies $\hat{\lambda}_k\sim|D|^{-1/(2r+s)}$, which is theoretically optimal.

\section{Numerical Results}\label{Sec.Experiments}

In this section, we report experimental results to verify our theoretical assertions and show the power of DKRR in tackling distributively stored  time series. The numerical results are divided into three parts: verifying  theoretical assertions of DKRR for distributively stored data in a sequential manner, verifying  theoretical assertions of DKRR for distributively stored data in a parallel manner and showing the performance of DKRR in  a real world application.   The environment of our experiments is: Matlab R2021a with Intel(R) Xeon(R) Gold 5119T CPU @1.90GHz, Windows 10.

Throughout this section, we use  one-step prediction to show the power of the proposed learning algorithms. To be detailed, that is, given training data $D_{tr}:=\{x_i^{tr}\}$ and testing data $D_{te}:=\{x_j^{te}\}$, we get an estimator $f_{D_{tr}}$ based on  $D_{tr}$, and then produce an approximation of $x_j^{te}$ as $f_{D_{tr}}(x_{j-1}^{te})$, where $x_{0}^{te}=x_{|D_{tr}|}^{tr}$ denotes the last samples in the training set. In this way, the testing error recorded in our experiments is the average of all the one-step prediction error, i.e., $\sqrt{\frac{1}{|D_{te}|}\sum_{j=1}^{|D_{te}|} |f_{D_{tr}}(x_{j-1}^{te})-x_j^{tr}|^2}$.
More specific ways in which data are generated will be explained separately in each experiment.

\subsection{DKRR for sequentially distributively stored  time series}
In this toy simulation, we adopt  the widely used Wendland  kernel \cite{Chang2017}  in DKRR as follow:
\begin{eqnarray*}
   K(x,x')=\left\{
\begin{array}{cc}
 (1-\|x-x'\|_2)^4(4\|x-x'\|_2+1)\ &\text{if}\ 0<\|x-x'\|_2\leq1 \\
 0 &\ \text{if}\ \|x-x'\|_2>1.
\end{array}
\right.
\end{eqnarray*}
It is easy to check that $K$ is a Mercer kernel.
We consider four time series, including different generating mechanisms   and different types of noise.
\begin{eqnarray}
\mbox{Mechanism 1}:\qquad x_{t+1} &=& 0.5\sin(x_{t})+\varepsilon_t,  \label{f_1}\\
\mbox{Mechanism 2}:\qquad x_{t+1} &=&\cos(x_{t})\sin(x_{t-1})+\varepsilon_t, \label{f_2}
\end{eqnarray}
where $\varepsilon_t$ is the independent noise satisfying either  $\varepsilon_t\sim\mathcal{N}(0,0.4^2)$ or  $\varepsilon_t\sim\mathcal{U}(-0.7,0.7)$, and $x_1$ for (\ref{f_1}) and $x_1,x_2$ for (\ref{f_2}) are generated from $\mathcal{U}(0,1)$. Then it can be found in \cite{Alquier2013} that the above four time series are  geometrical $\alpha$-mixing sequences. Furthermore, it can be easily deduced from the definition of $f_\rho$ that  regression functions for Mechanism 1 with both Gaussian noise and uniform noise is $f_\rho(x)=0.5\sin x$ and for
 Mechanism 2 with both Gaussian noise and uniform noise is $f_\rho(x)=\cos(x_{t})sin(x_{t-1})$.

Our purposes in this simulation are three fold. The first one is to study the role of number of local machines in DKRR  to verify Theorem \ref{Theorem:KRR-derr-dec}. The second one is to verify Corollary \ref{Corollary:KRR-derr-dec} via showing the relation between test error and size of samples. The last purpose is to exhibit the power of DKRR on predicting time series.
In this way, we adopt three criteria for comparison. The first criterion is the \emph{global mean squared error} (GMSE) which is the mean square error (MSE) of running  KRR on all samples in a batch mode. We regard it as   a baseline for  analysis. The second criterion is the \emph{average error} (AE) which is the MSE of DKRR. AE describes the learning performance of DKRR.  The last one is the \emph{last local error} (LLE), which is the MSE of running KRR on the data subset stored on the last local machine. Since the time series are distributively stored in a  sequential manner, LLE reflects the learning performance of running KRR on data subset of size $N/m$ in one single machine. For the sake of brevity, we denote such a KRR to be $DKRR_{last}$. It should be also highlighted that LLE also provides a reference for DKRR to show the power of distributed learning.

\subsubsection{Relation between MSE and number of local machines}

In this simulation, we assume that there are  $N=5000$ training samples generated via (\ref{f_1}) or (\ref{f_2})  distributively stored on $m$ local machines in a sequential manner.  We  also generate 100 testing samples according to (\ref{f_1}) or (\ref{f_2}).
Since our aim is to estimate the regression function, we remove the noise term to record the test error.
 {The number of local machines  varies from $\{1,2,4,5,10,20,25,30,40,50,100\}$. The experiments are repeated 20 times, and we record the logarithm of MSE when the number of local machines  varies from $\{1,2,4,5,10,20\}$ in the subgraphs of corresponding figures.}
Regularization parameters in all experiments are selected by grid search. For a given data, the regularization parameters for DKRR with different $m$ are the same and are the same as the optimal regularization parameter  $\lambda_{opt}$  for KRR, as shown in  Table \ref{Tab:lambda}. The numerical results are reported in Figure \ref{select_m1}.
\begin{table}[htbp]
\caption{$\lambda_{opt}$   for different time series}\label{Tab:lambda}
\begin{center}
\begin{tabular}{|c|c|c|}
\hline
&Uniform&Gaussian\\
\hline
Mechanism 1&0.0024&0.0004\\
\hline
Mechanism 2&0.001&0.0007\\
\hline
\end{tabular}
\end{center}
\end{table}

\begin{figure}[htbp]
\centering
\subfigure[Gaussian noise on (\ref{f_1})]{\includegraphics[scale=0.25]{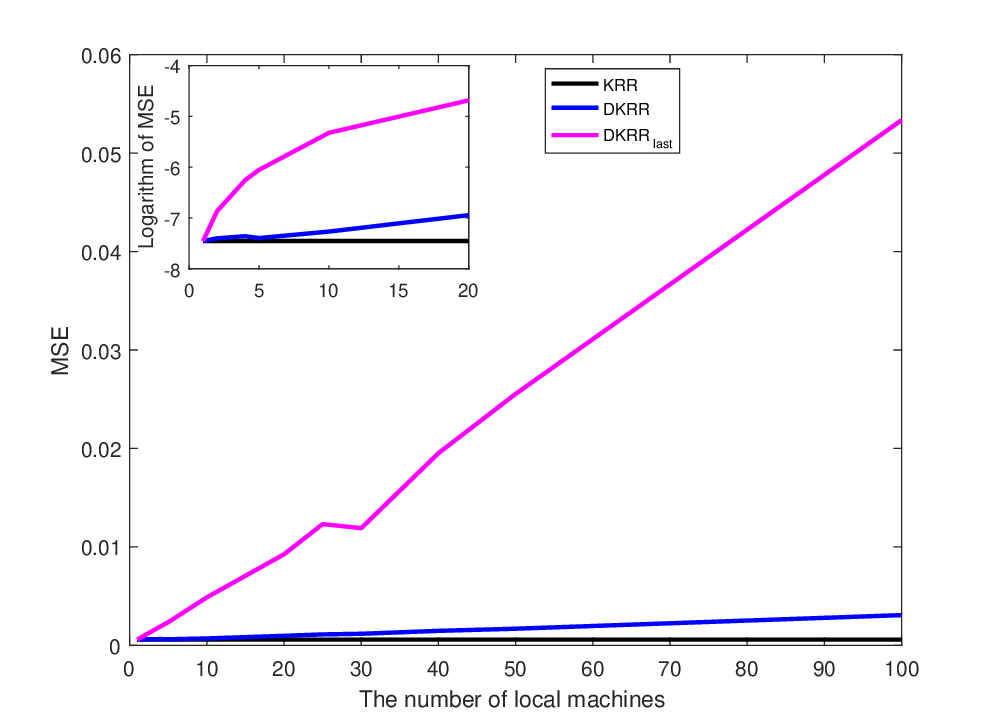}\label{seq_gau_m1}}
\subfigure[Uniform noise on (\ref{f_1})]{\includegraphics[scale=0.25]{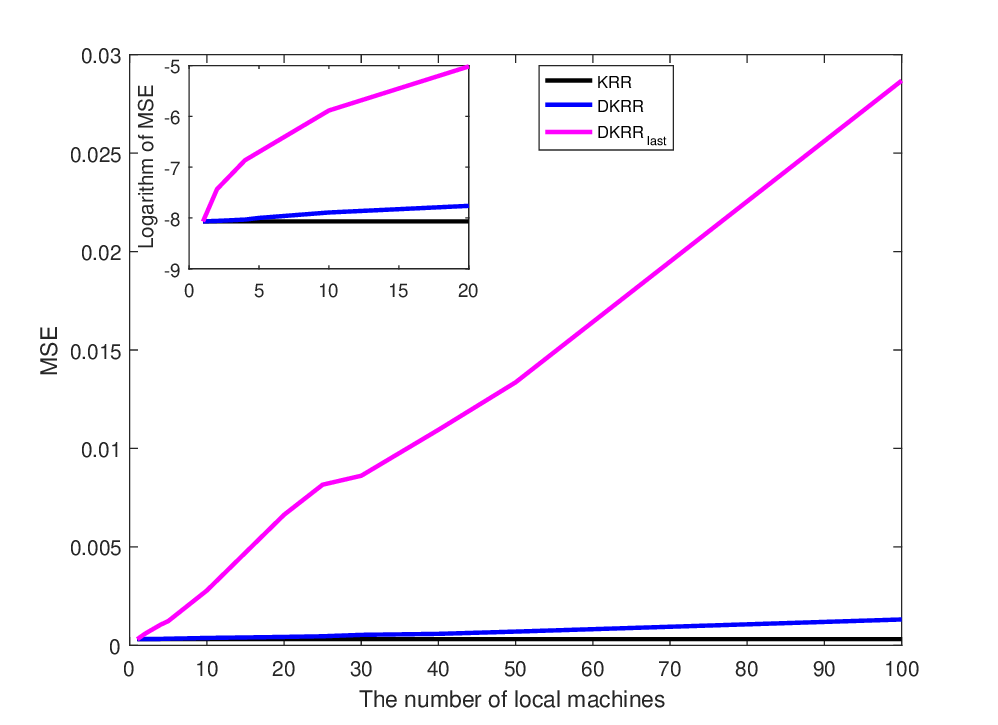}\label{seq_uni_m1}}\\
\subfigure[Gaussian noise on (\ref{f_2})]{\includegraphics[scale=0.25]{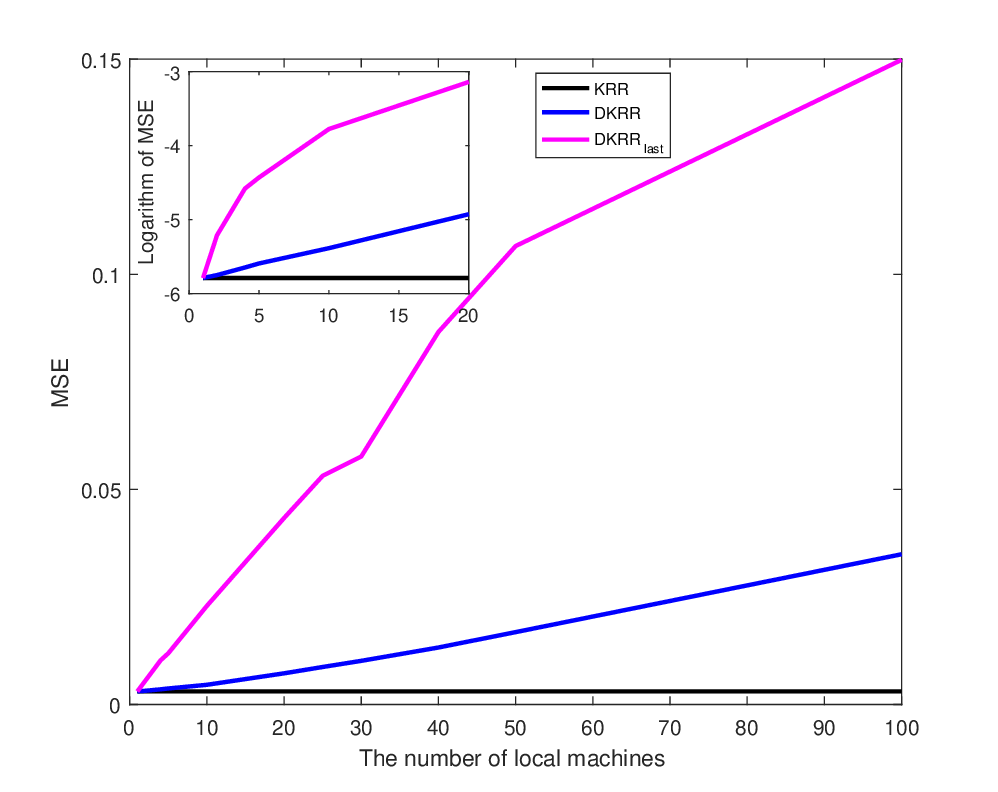}\label{seq_gau_m2}}
\subfigure[Uniform noise on  (\ref{f_2}) ]{\includegraphics[scale=0.25]{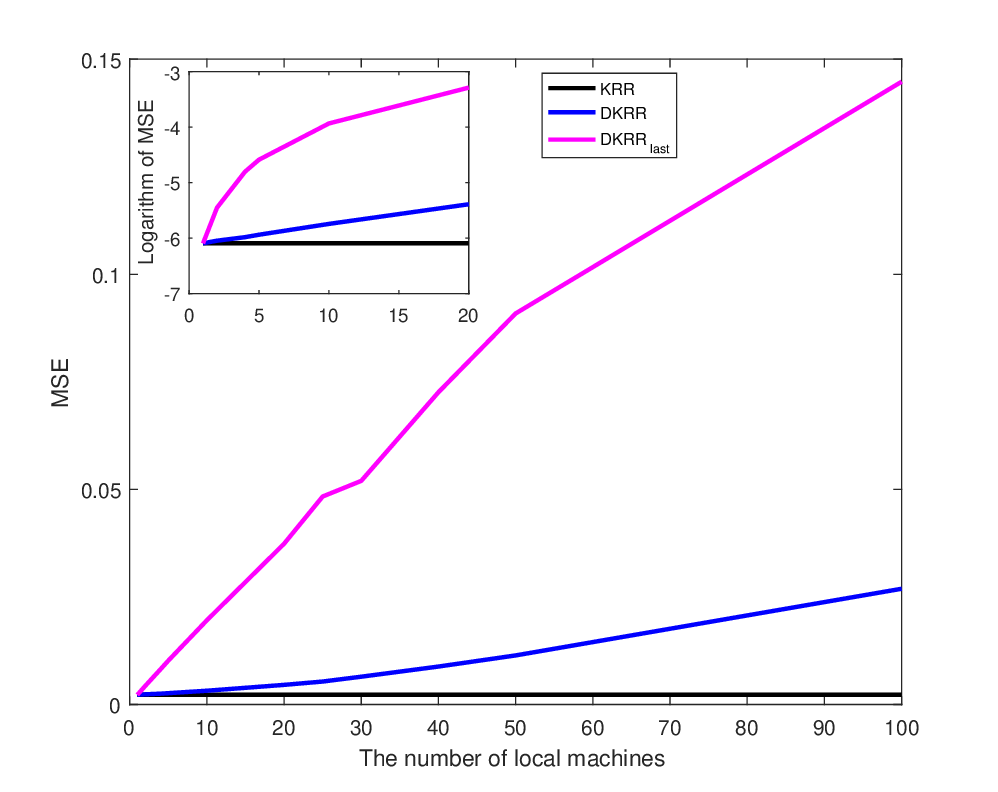}\label{seq_uni_m2}}
\caption{Relation between MSE and the number of local machines.}\label{select_m1}
\end{figure}

Figure \ref{select_m1} exhibits the relation between MSE for DKRR and the number of local machines. Based on Figure \ref{select_m1},  three conclusions can be drawn: 1) AEs are always comparable to GMSEs when the number of local machines is not too large. This verifies the theoretical statement in Corollary \ref{Corollary:KRR-derr-dec}, showing that   DKRR performs similarly as  KRR for strong mixing sequences when $m$ is not so large.
2) There exists an upper bound of $m$, larger than which DKRR will degrade  the performance of KRR, just as shown in (\ref{condition-m}) purports to show.    3) LLE curve  increases dramatically and is far from the AE curves, implying that DKRR can significantly improve prediction accuracy in comparison to prediction  produced by a single local machine. All these show that for tackling sequentially and distributively stored time series, DKRR can finish the learning task much better than KRR with arbitrary data subset and performs similarly as KRR with all data, provided $m$ is not so large.


\subsubsection{Relation between MSE and number of samples}
In this simulation, we test the generalization performance of DKRR with the number of training samples varying from 5000 to 100000 (the interval is 5000) when $m=10,20,50,100$ respectively. Other experimental setting is the same as that in the previous simulation. For the sake of brevity, we only report the result for Mechanism 2 in (\ref{f_2}) in  Figure \ref{select_n1}.
\begin{figure}[htbp]
\centering
\subfigure[Gaussian noise]{\includegraphics[scale=0.295]{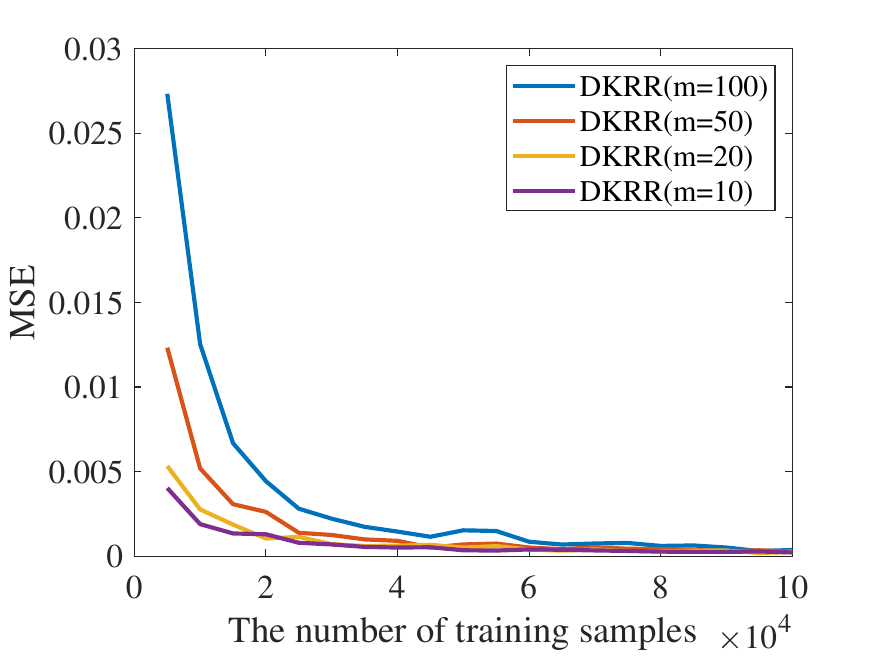}\label{seq_gau_n}}
\subfigure[Uniform noise]{\includegraphics[scale=0.295]{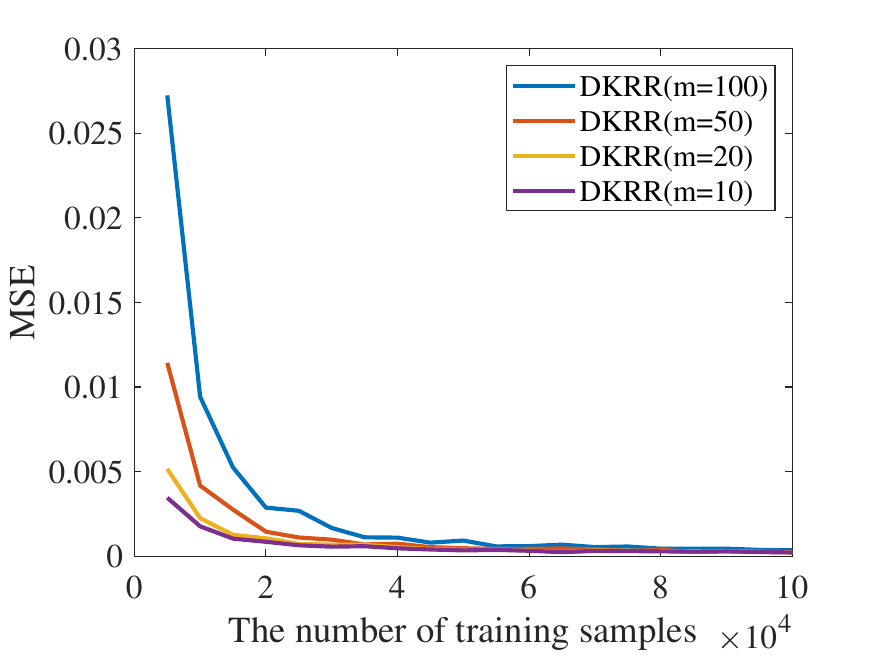}\label{seq_uni_n}}
\caption{Relation between MSE and number of samples for (\ref{f_2}).}\label{select_n1}
\end{figure}

From Figure \ref{select_n1} we see that, as the number of training samples increases, MSE  for DKRR decreases.  This means that DKRR can get better performance with more training samples. 
For each fixed number of local machines, there exists a lower bound of training samples to achieve a relatively small error. Obviously, the lower bound increases as the number of local machines increases. This also confirms the conclusion in Corollary \ref{Corollary:KRR-derr-dec}. Additionally, using DKRR with larger $m$ can significantly shorten the computational  time. Specifically, we run DKRR on Matlab R2021a and the running time for $m = 10, 20, 50, 100$ is about $80, 40, 20, 5$ hours respectively. In other words, DKRR can help to resolve  time-consuming problem for large-scale time series.

\subsubsection{Trends of DKRR in forecasting time series}
We predict the time series generated by (\ref{f_1}) with Gaussian noise. In this simulation, we set $m=10$ in DKRR. The total number of training samples and test samples are 5000 and 50 respectively. Figure~\ref{prediction_w/onoise} shows prediction of (\ref{f_1}) without label noise and Figure~\ref{prediction_wnoise} shows prediction of { (\ref{f_1})} with label noise, where ``Label'' denotes the real trend generated by (\ref{f_1}). As shown in Fig.~\ref{prediction_w/onoise} and Fig.~\ref{prediction_wnoise}, DKRR can successfully capture the trend information, especially when the effects of noise are neglected. This shows that, DKRR is a powerful tool to predict the trend of time series, provided they are distributively in a sequential manner.

\begin{figure}[htbp]
\centering
\subfigure[ label without noise]{\includegraphics[scale=0.295]{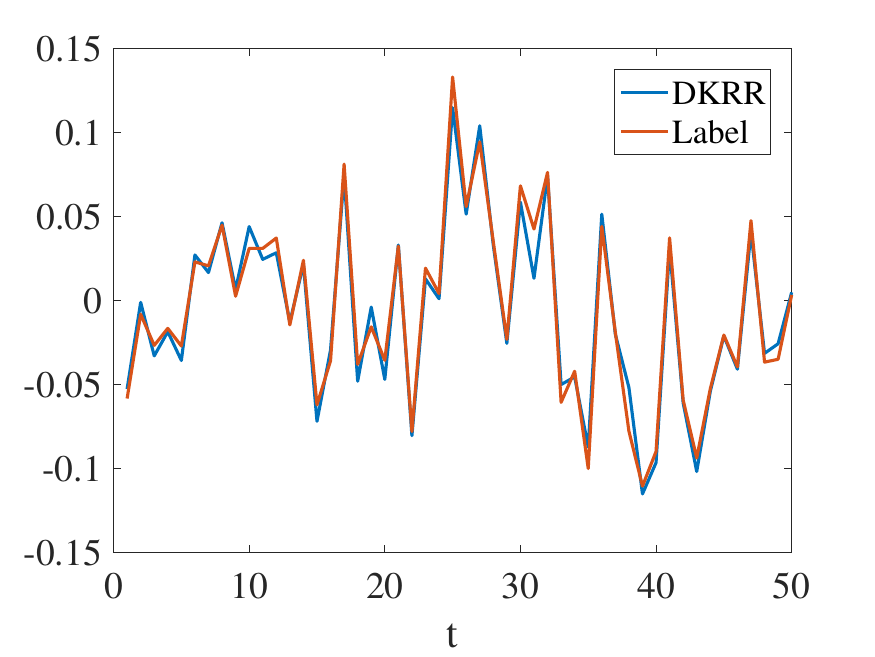}\label{prediction_w/onoise}}
\subfigure[ label with noise]{\includegraphics[scale=0.295]{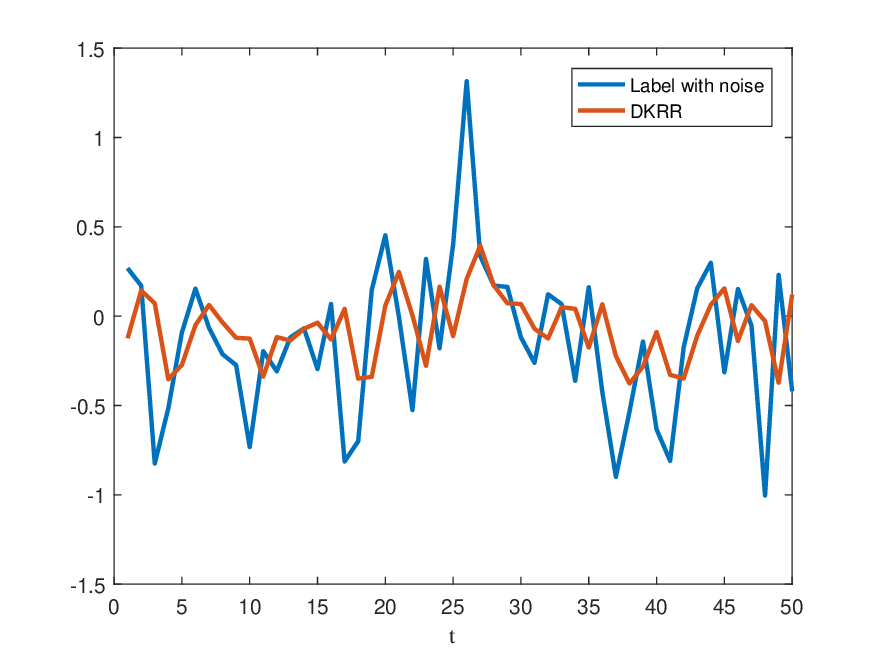}\label{prediction_wnoise}}
\caption{Prediction of time series generated by (\ref{f_1})}\label{prediction}
\end{figure}

\subsection{DKRR for parallel distributively stored   time series}
The kernel, time series generating mechanism  and regularization parameter of this simulation are the same as those in the previous subsection. The only difference is that the time series in this simulation are distributively stored in a parallel manner.
Denote by $N_{tr}$ and $N_{te}$ the number of total training samples and testing samples generated by a local machine, respectively. For each local machine, we generate $[N_{tr}/m]+N_{te}+1$ samples via (\ref{f_1}) or $[N_{tr}/m]+N_{te}+3$ samples via (\ref{f_2}). The training samples for $k$-th local machine are: $\{x_{t,k},x_{t+1,k}\}_{t=1}^{[N_{tr}/m]}$ for (\ref{f_1}) with $x_{1,k}$ drawn randomly according to $\mathcal U(0,1)$ and  $\{[x_{t,k},x_{t+1,k}],x_{t+2,k}\}_{t=1}^{[N_{tr}/m]}$ for (\ref{f_2}) with $x_{1,k},x_{2,k}$ randomly drawn according to $\mathcal U(0,1)$. We can construct a test set containing $N_{te}$ samples: $ D_{te}^k\triangleq \{x_{t,k},x_{t+1,k}-\varepsilon_t\}_{t=[N_{tr}/m]+1}^{[N_{tr}/m]+N_{te}}$ for (\ref{f_1}), $D_{te}^k\triangleq \{[x_{t,k},x_{t+1,k}],x_{t+2,k}-\varepsilon_t\}_{t=[N_{tr}/m]+2}^{[N_{tr}/m]+N_{te}+1}$ for (\ref{f_2}). And the test samples are the aggregation of test set of all local machines: $D_{te}\triangleq\bigcup_{k=1}^m\{D_{te}^k\}$.

Similar as LLE in the previous subsection, we adopt the \emph{minimum local error} (MLE) which  is the smallest  MSE for running KRR on each data subset. We denote it by $DKRR_{min}$. Actually, $DKRR_{min}$ is near to the widely used real-world setting: samples are stored in different companies or institutions and they do not want to share their data because of some privacy policies. Thus each company or institution has to predict independently. In such a way, $DKRR_{min}$ reflects the best prediction of individual company with their own data  and presents a bottleneck for such a learning scheme.

\subsubsection{Relation between MSE and number of local machines}

In this simulation, the number of local machines  varies from $\{1,2,4,5,10,20,25,40,50,100\}$, and we also record
the logarithm of MSE when the number of local machines  varies from $\{1,2,4,5,10,20\}$ in the subgraphs of corresponding
figures.
We set $N_{tr} =5000$, $N_{te}=100$  and the test samples are $\bigcup_{k=1}^m\{D_{te}^k\}$, that is, the total numbers of training samples is fixed as 5000 and testing samples is $100m$ for different value of $m$. Other numerical settings are the same as its sequential counterpart in the previous subsection. The numerical results are shown in Fig.~\ref{select_m}.

\begin{figure}[htbp]
\centering
\subfigure[Gaussian noise on (\ref{f_1})]{\includegraphics[scale=0.25]{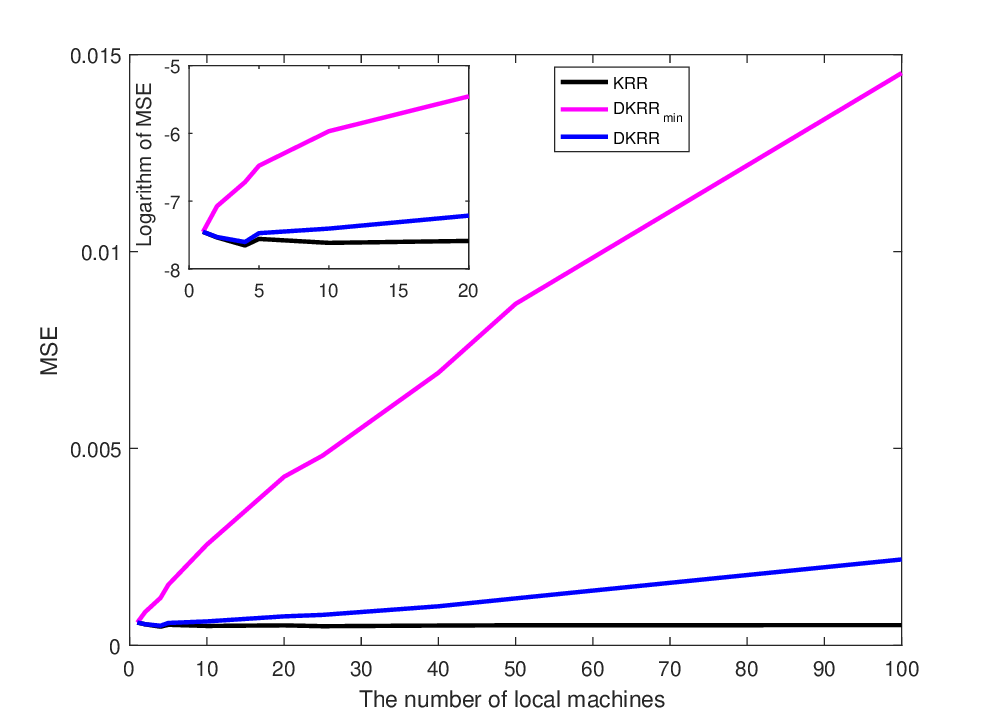}\label{para_gau_m1}}
\subfigure[Uniform noise on (\ref{f_1})]{\includegraphics[scale=0.25]{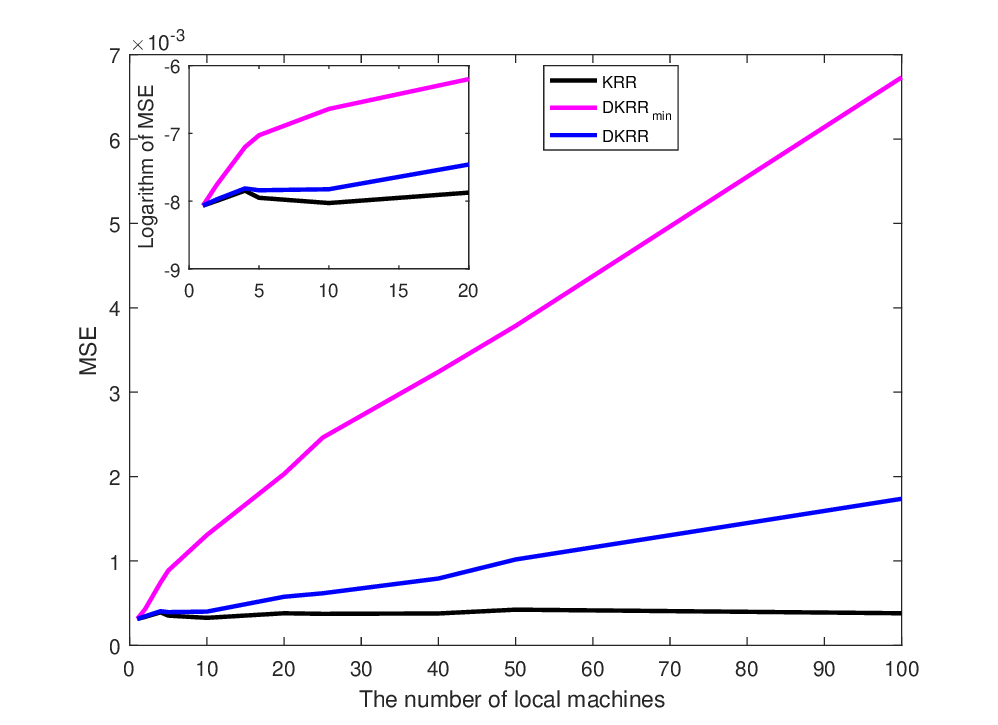}\label{para_uni_m1}}\\
\subfigure[Gaussian noise on (\ref{f_2})]{\includegraphics[scale=0.25]{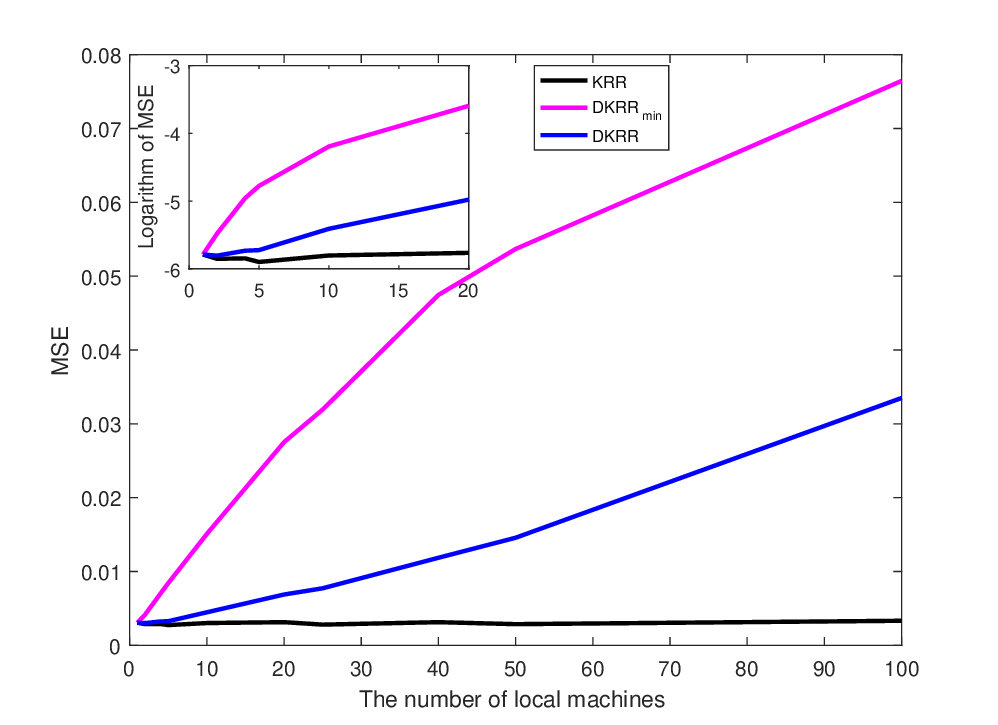}\label{para_gau_m2}}
\subfigure[Uniform noise on (\ref{f_2})]{\includegraphics[scale=0.25]{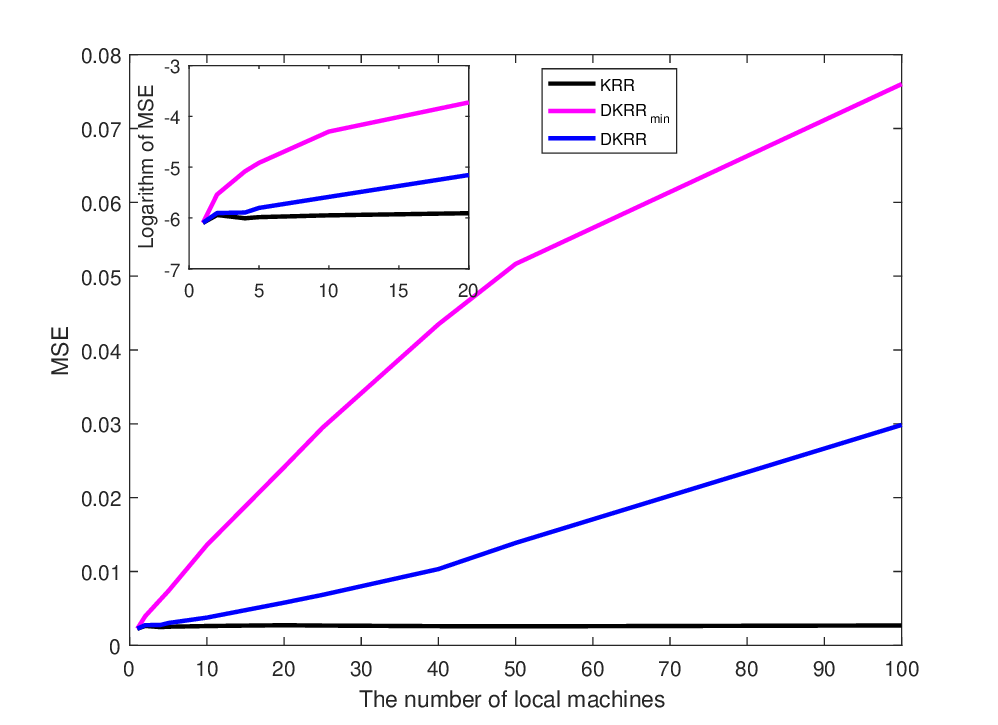}\label{para_uni_m2}}\\
\caption{Relation between MSE and number of local machines}\label{select_m}
\end{figure}

From Fig.~\ref{select_m}, we can draw the following two conclusions similar as its sequential counterpart in the previous subsection: 1) DKRR performs similar as KRR with whole data provided $m$ is not so large by showing similar AEs and GMSEs.
2) MLE curves increase sharply and are far from the AE curves when $m>1$, which means that the predictions will be inaccurate if a company only uses its own data to make predictions in real world applications. However, in order to protect privacy, companies cannot share data with each other, which increases the obstacles for companies to cooperate to improve prediction accuracy.  DKRR solves this dilemma significantly. We can see that AEs can always be compared with GMSEs when the number of local machines is not too large. This not only verifies the theoretical result of Corollary \ref{Corollary:KRR-derr-dec} once again, but also shows that we can effectively use the data information of several companies by using the DKRR under the premise of protecting privacy and not sharing their own data, and achieve the same prediction accuracy as the companies completely share their  data for predicting.

\subsubsection{Relation between MSE and number of samples}

{  In this simulation, we test the generalization performance of DKRR with total number of training samples varying from 5000 to 100000 (the interval is 5000) and fix numbers of testing samples as 100. Here, we set $N_{te}=100$ and the test samples are $D_{te}^1$. The other numerical settings are the same as its sequential counterpart in the previous subsection. For the sake of brevity, we only report the result for Mechanism 2 in Fig.~\ref{select_n}. }

\begin{figure}[htbp]
\centering
\subfigure[Gaussian noise]{\includegraphics[scale=0.295]{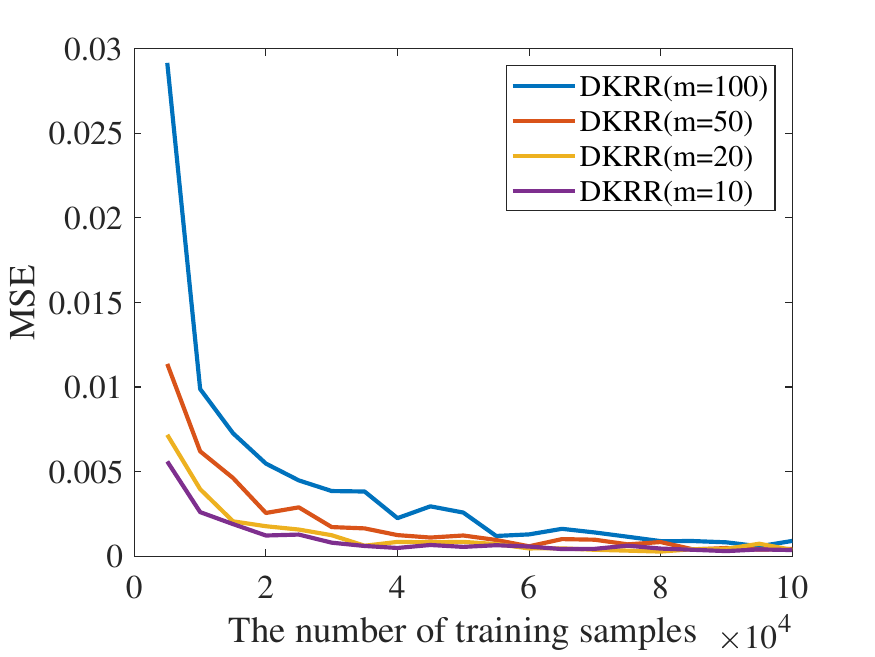}\label{para_gau_n}}
\subfigure[Uniform noise]{\includegraphics[scale=0.295]{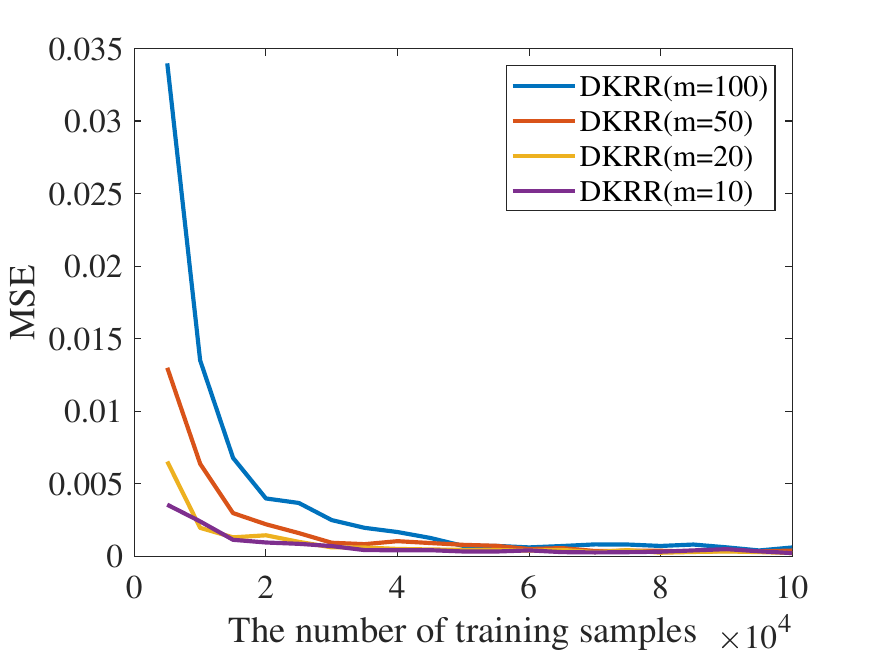}\label{para_uni_n}}
\caption{Relation between MSE and number of samples for (\ref{f_2}).}\label{select_n}
\end{figure}

From Fig.~\ref{select_n} we see that, for each fixed number of local machines, MSE for DKRR  decreases when the numbers of training samples   increases, and there exists a lower bound of training samples to achieve a relatively small error. Obviously, the lower bound increases as the number of local machines increases. This also verifies the theoretical result in Corollary \ref{Corollary:KRR-derr-dec}.

\subsubsection{Trends of DKRR in forecasting time series}

{ In this simulation, we predict the time series generated by (\ref{f_1}) with Gaussian noise and set $m = 10$ in DKRR. The total number of training samples and test samples are 5000 and 50 respectively. Here, we set $N_{te}=50$ and the test samples are $D_{te}^1$. The other numerical settings are the same as its sequential counterpart in the previous subsection and the numerical results are shown in Fig.~\ref{prediction_para}.
From figures we see that, DKRR can successfully extract the trend information, especially when the influence of noise is ignored. This shows that DKRR is still a powerful tool for predicting the trend of a time series when they are distributively in a parallel manner.
}

\begin{figure}[htbp]
\centering
\subfigure[label without noise]{\includegraphics[scale=0.295]{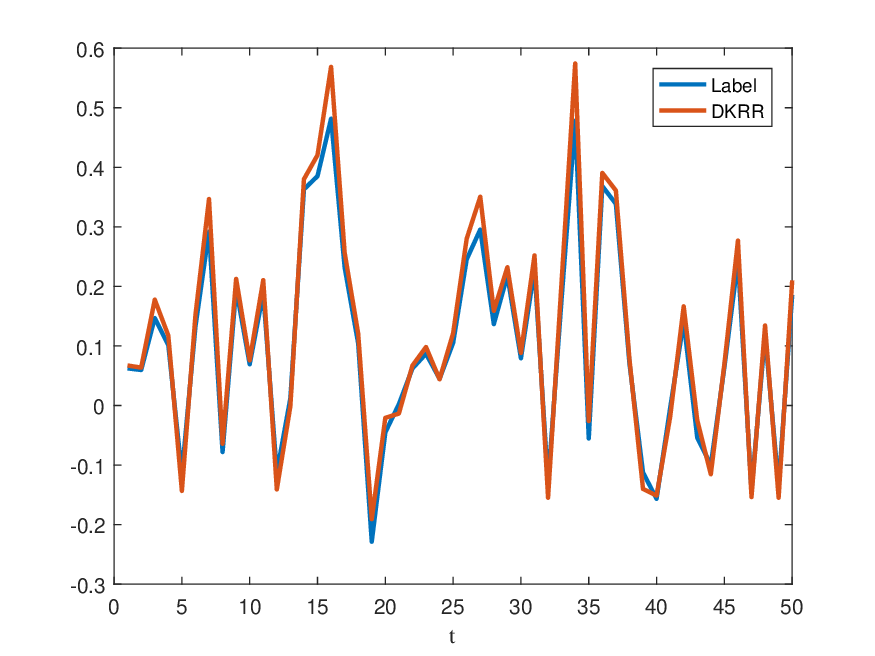}\label{prediction_para_w/onoise}}
\subfigure[label with noise]{\includegraphics[scale=0.295]{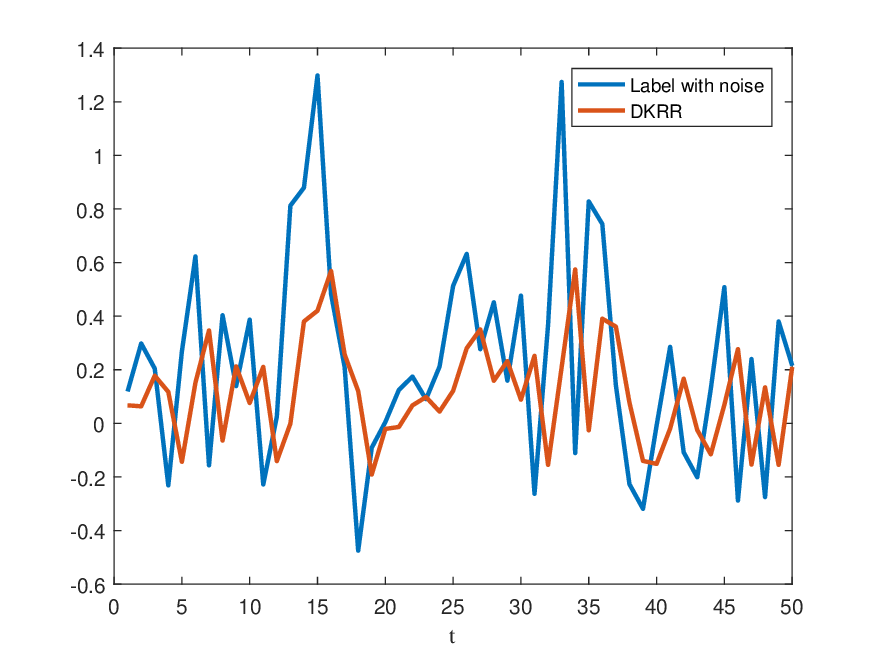}\label{prediction_para_wnoise}}
\caption{Prediction of time series generated by (\ref{f_1})}\label{prediction_para}
\end{figure}

\subsection{Real world applications}
{  In this part, we focus on Europe Brent Spot Prices data (\url{https://datahub.io/core/oil-prices}) from September 12, 2018 to August 28, 2020 and the data is daily recorded. There are 500 items in total, and we predict the last 100 samples by DKRR on setting m = 10 or 100. We use kernel function
\begin{eqnarray*}
K(x,x')&=&1+min(x,x')
\end{eqnarray*}
in the experiment.  The results are shown in Fig.~\ref{Brant_prediction}. According to the curves in the figure, we can see that the predicted trend has little difference from the real trend. When the value of $m$ is relatively large, we can still accurately predict the trend through DKRR. All these show the power of DKRR in predicting distributively stored time series.}

\begin{figure}[htbp]
\centering
\subfigure[The result of $m=10$]{\includegraphics[scale=0.295]{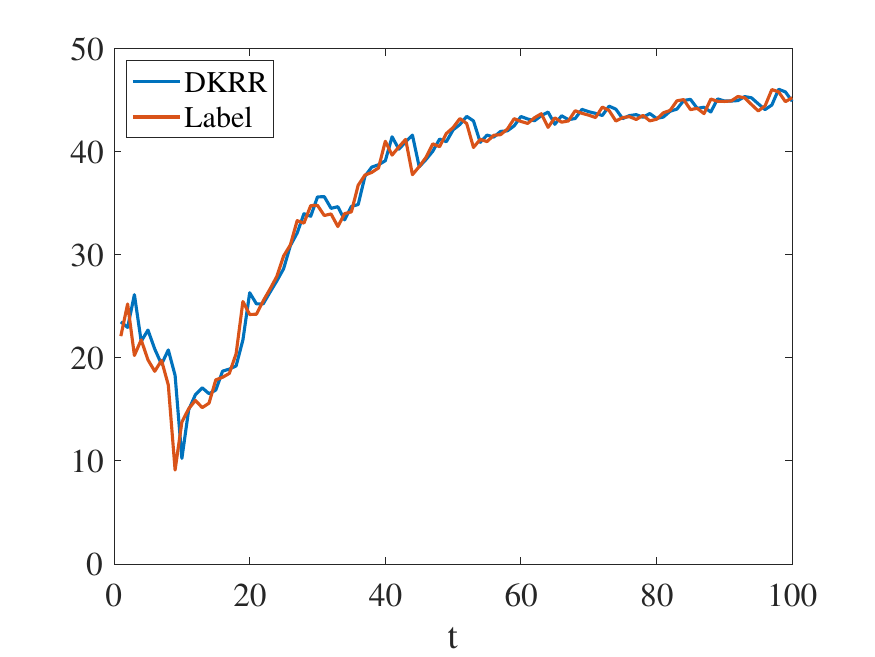}\label{prediction_daily_m10}}
\subfigure[The result of $m=100$]{\includegraphics[scale=0.295]{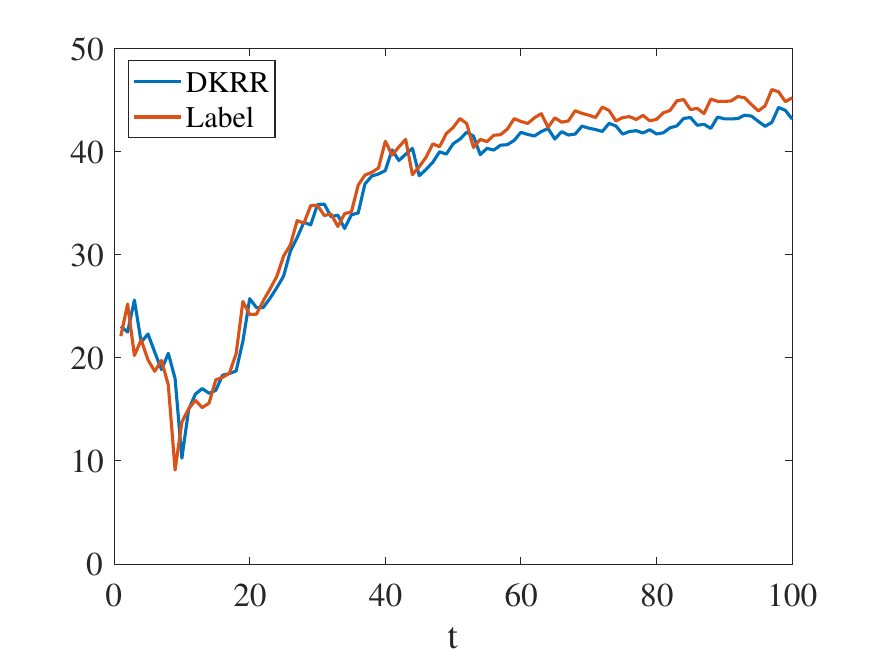}\label{prediction_daily_m100}}
\caption{The predicting trend of Europe Brent Spot Prices.}\label{Brant_prediction}
\end{figure}

\section{Step-stones for Analysis }\label{Sec:tools}
Since analysis technique based on Bernstein-type concentration inequality \cite{Modha1996} cannot provide optimal learning rates, we turn to utilizing the integral operator approach, just as \cite{Sun2010} did. Our main novelties in the proof are    the recently developed integral operator technique in \cite{Blanchard2016,Lin2017,Guo2017} and the following covariance inequality for Hilbert-valued strong mixing sequences   \cite[Lemma 2.2]{Dehling1982}, with which we succeed in deriving tight bounds for operator differences in Lemmas \ref{Lemma:Operator-difference}-\ref{Lemma:Operator-difference-2}.

\begin{lemma}\label{Lemma:tool1}
Let $\xi$ and $\eta$ be random variables with values in a separable Hilbert space $\mathcal H$.
Let $\mathcal J$ and $\mathcal K$ be two measurable $\sigma$-fields generated by $\xi$ and $\eta$, respectively.    $u=v=\infty$, $t=1$, then
\begin{equation}\label{tool-main}
    |E\langle \xi,\eta\rangle_{\mathcal H} - \langle E\xi,E\eta\rangle_{\mathcal H}|\leq  15(\alpha(\mathcal J,\mathcal K))^{1/t}\|\xi\|_u \|\eta\|_v,
\end{equation}
where $\|\xi\|_u$ denotes the $u$-th moment as $\|\xi\|_u=(E\|\xi\|_{\mathcal H}^u)^{1/u}$ for $1\leq u<\infty$ and $\|\xi\|_\infty=\sup\|\xi\|_{\mathcal H}$  and $\alpha(\mathcal J,\mathcal K)$ is defined by (\ref{alpha}).
\end{lemma}

\subsection{Integral operator approach for strong mixing sequences}

 Let $S_{D}:\mathcal H_K\rightarrow\mathbb R^{|D|}$ be the sampling
operator \cite{Smale2007} defined by
$$
         S_{D}f:=(f(x))_{(x,y)\in D}.
$$
 Its scaled adjoint $S_D^T:\mathbb R^{|D|}\rightarrow
\mathcal H_K$  is
given by
$$
       S_{D}^T{\bf c}:=\frac1{|D|}\sum_{i=1}^{|D|}c_iK_{x_i},\qquad {\bf
       c}:=(c_1,c_2,\dots,c_{|D|})^T
       \in\mathbb
       R^{|D|}.
$$
Define
$$
         L_{K,D}f:=S_D^TS_Df=\frac1{|D|}\sum_{(x,y)\in
         D}f(x)K_x.
$$
Then, it can be  found in  \cite{Lin2017} that
\begin{equation}\label{operator KRR}
    f_{D,\lambda}=\left(L_{K,D}+\lambda
    I\right)^{-1}S_{D}^Ty_D
\end{equation}
and
\begin{equation}\label{operator DKRR}
    \overline{f}_{D,\lambda}=\sum_{j=1}^m\frac{|D_j|}{|D|}
    \left(L_{K,D_j}+\lambda
    I\right)^{-1}S_{D_j}^Ty_{D_j},
\end{equation}
where $y_D:=(y_1,\dots,y_{|D|})^T$.

Our first tool is to derive upper bounds for operator differences
$\|L_K-L_{K,D}\|$ and $\|(L_K+\lambda I)^{-1/2}(L_K-L_{K,D})\|$ in the following lemma, whose proof will be given in Appendix.

 \begin{lemma}\label{Lemma:Operator-difference}
  For an $\alpha$-mixing sequence $\{x_i\}$ and arbitrary $\delta>0$, we have
\begin{equation}\label{operator-d1}
       E[\| (L_K-L_{K,D})\|^2]
       \leq
       \frac{\kappa^4}{|D|}\left(1+30\sum_{\ell=1}^{|D|}\alpha_\ell\right)
\end{equation}
and
\begin{eqnarray}\label{operator-d2}
   &&E[\|\left(L_K+\lambda I\right)^{-1/2} (L_K-L_{K,D})\|^2] \nonumber\\
   &\leq&
   \frac{\kappa^2\mathcal N(\lambda) }{|D|}
   +
   15 \kappa^{\frac{4(\delta+1)}{\delta+2}}(\mathcal N(\lambda))^{\frac{2}{\delta+2}} \lambda^{-\frac{\delta}{\delta+2}} \frac1{|D|}\sum_{\ell=1}^{|D|}(\alpha_\ell)^{\frac\delta{2+\delta}}.
\end{eqnarray}
\end{lemma}

As shown in \cite{Guo2017,Lin2018CA}, the product $\|(L_{K,D}+\lambda I)^{-1}(L_{K}+\lambda I)\|$ plays a  crucial role in deriving optimal learning rates of  KRR and DKRR for i.i.d. samples. In the following lemma, we adopt the recently developed second order decomposition for operator differences in \cite{Lin2017,Guo2017} to present an upper bound of $\|(L_{K,D}+\lambda I)^{-1}(L_{K}+\lambda I)\|$ for strong mixing sequences, whose proof is also postponed to Appendix.

\begin{lemma}\label{Lemma:Operator-product}
  For an $\alpha$-mixing sequence $\{x_i\}$ and arbitrary $\delta>0$, there holds
 \begin{eqnarray*}
    &&E\left[\|(L_{K,D}+\lambda I)^{-1}(L_{K}+\lambda I)\|\right]\\
    &\leq&
     \frac{2\kappa^2\mathcal N(\lambda) }{|D|\lambda}+
    30 \kappa^{\frac{4(\delta+1)}{ \delta+2}}(\mathcal N(\lambda))^{\frac{2}{\delta+2}} \lambda^{-\frac{2\delta+2}{\delta+2}} \frac{1}{|D|}\sum_{\ell=1}^{|D|}(\alpha_\ell)^{\frac\delta{2+\delta}}
     +2.
\end{eqnarray*}
\end{lemma}

With the help of   Lemma \ref{Lemma:tool1}, we provide  our final tool which focuses on tight bounds for the difference between functions $L_Kf_\rho$ and $S_D^Ty_D$ in the following lemma.

\begin{lemma}\label{Lemma:Operator-difference-2}
Let  $|y_i|\leq M$ almost surely. For an $\alpha$-mixing sequence $\{(x_i,y_i)\}$ and arbitrary $\delta>0$, we have
\begin{equation}\label{operator-d-2-1}
    E[\| (L_Kf_\rho-S_D^Ty_D)\|_K^2]
    \leq \frac{M^2\kappa^2}{|D|}\left(1+30\sum_{\ell=1}^{|D|} \alpha_{\ell}\right)
\end{equation}
and
 \begin{eqnarray}\label{operator-d-2-2}
   &&E[\|\left(L_K+\lambda I\right)^{-1/2} (L_Kf_\rho-S_D^Ty_D)\|_K^2]\\
   &\leq&
   \frac{M^2\mathcal N(\lambda) }{|D|}
   +
   15 \kappa^{\frac{2\delta}{\delta+2}}M^2(\mathcal N(\lambda))^{\frac{2}{\delta+2}} \lambda^{-\frac{\delta}{\delta+2}}\frac1{|D|}\sum_{\ell=1}^{|D|}(\alpha_\ell)^{\frac\delta{2+\delta}}.\nonumber
\end{eqnarray}
\end{lemma}

\subsection{Error decompositions for KRR}
We present different error decompositions for KRR under the $\mathcal H_K$ norm and $L_{\rho_X}^2$ norm. Define the population version of $f_{D,\lambda}$ to be
\begin{equation}\label{best-app}
   f_\lambda=(L_K+\lambda I)^{-1}L_Kf_\rho.
\end{equation}
Then  (\ref{operator KRR}) and  (\ref{best-app}) yield the following two decompositions.
\begin{eqnarray}\label{Op-KRR-d1}
    &&f_{D,\lambda}-f_\lambda=(L_{K,D}+\lambda I)^{-1}S_D^T y_D-(L_K+\lambda I)^{-1}L_Kf_\rho    \\
    &=&
    ((L_{K,D}+\lambda I)^{-1}- (L_K+\lambda I)^{-1})S_D^T y_D \nonumber\\
    &+&
    (L_K+\lambda I)^{-1}(S_D^T y_D-L_Kf_\rho)\nonumber\\
    &=&
    (L_K+\lambda I)^{-1}(L_K-L_{K,D})f_{D,\lambda}+(L_K+\lambda I)^{-1}(S_D^T y_D-L_Kf_\rho)   \nonumber
\end{eqnarray}
and
\begin{eqnarray}\label{Op-KRR-d2}
    &&f_{D,\lambda}-f_\lambda=(L_{K,D}+\lambda I)^{-1}S_D^T y_D-(L_K+\lambda I)^{-1}L_Kf_\rho \nonumber \\
    &=&
    (L_{K,D}+\lambda I)^{-1}(S_D^T y_D-L_Kf_\rho)\nonumber\\
    &+&(L_{K,D}+\lambda I)^{-1}-(L_K+\lambda I)^{-1})L_Kf_\rho\nonumber\\
    &=&
    (L_{K,D}+\lambda I)^{-1}(S_D^T y_D-L_Kf_\rho)\nonumber\\
    & +&
    (L_{K,D}+\lambda I)^{-1}(L_K-L_{K,D})f_\lambda.
\end{eqnarray}

Denote
\begin{eqnarray}
    \mathcal P_{|D|,\lambda}
    &:=& E[\|(L_K+\lambda I)^{-1/2}(L_K-L_{K,D})\|^2] \label{Def.P}\\
    \mathcal Q_{|D|,\lambda}
    &:=& E\left[\|(L_{K,D}+\lambda I)^{-1}(L_{K}+\lambda I)\|\right] \label{Def.Q}\\
    \mathcal R_{|D|,\lambda}
    &:=&
    E[\|\left(L_K+\lambda I\right)^{-1/2} (L_Kf_\rho-S_D^Ty_D)\|_K^2] \label{Def.R}\\
    \mathcal S_{|D|}
    &:=&
    E[\| (L_K-L_{K,D})\|^2]\label{Def.S}\\
    \mathcal T_{|D|}
    &:=&
    E[\| (L_Kf_\rho-S_D^Ty_D)\|_K^2] \label{Def.T}.
\end{eqnarray}
We have the following error decomposition for KRR under the $\mathcal H_K$ norm.
\begin{proposition}\label{Proposition:KRR-err-dec-HK}
Under  Assumption 5 with $1/2\leq r\leq 1$, we have
\begin{eqnarray}\label{error-dec-HK-prop}
    &&E[\|f_{D,\lambda}-f_\rho\|_K] \nonumber\\
    &\leq& \lambda^{r-1/2}{  \|h_\rho\|_\rho}
    +
   \lambda^{-1/2} \mathcal Q_{|D|,\lambda}^{1/2}(\mathcal R_{|D|,\lambda}^{1/2}+\kappa^{2r-1}\mathcal P_{|D|,\lambda}^{1/2}\|h_\rho\|_\rho).
\end{eqnarray}
\end{proposition}

\begin{proof}
It follows from the triangle inequality that
\begin{equation}\label{classic-error-dec-HK}
    \|f_{D,\lambda}-f_\rho\|_K\leq\|f_{D,\lambda}-f_\lambda\|_K+
    \|f_\lambda-f_\rho\|_K.
\end{equation}
But Assumption 5 with $ 1/2\leq r\leq 1$implies \cite{Chang2017}
\begin{equation}\label{app-error-HK}
     \|f_\lambda-f_\rho\|_K\leq\lambda^{r-1/2}\|h_\rho\|_\rho.
\end{equation}
Thus, it suffices to provide an upper bound for the variance $\|f_{D,\lambda}-f_\lambda\|_K$.
Due to the  Codes inequality \cite{Blanchard2016}
\begin{equation}\label{Codes inequality}
         \|A^\tau B^\tau\|\le\|AB\|^\tau, \qquad 0<\tau\leq 1
\end{equation}
 for positive operators $A$ and $B$,
 it follows from the Schwarz inequality, \eqref{Def.P}, \eqref{Def.Q}, \eqref{Def.R} and \eqref{Op-KRR-d2} that
\begin{eqnarray}\label{HK-variance-cl}
    &&E[\|f_{D,\lambda}-f_\lambda\|_K] \nonumber \\
    &\leq&
    E[\|(L_{K,D}+\lambda I)^{-1}(S_D^T y_D-L_Kf_\rho)\|_K]  \nonumber \\
    & +&
    E[\|(L_{K,D}+\lambda I)^{-1}(L_K-L_{K,D})f_\lambda\|_K]  \nonumber\\
    &\leq&
     \lambda^{  {-1/2}} \left(E[\|(L_{K,D}+\lambda I)^{-1/2}(L_{K}+\lambda I)^{1/2}\|^2]\right)^{1/2}    \nonumber \\
     &\times&
     \left\{\left(E[\|(L_K+\lambda I)^{-1/2}(S_D^T y_D-L_Kf_\rho)\|_K^2]\right)^{1/2}\right.   \nonumber\\
     &+&\left.\left(E[\|(L_K+\lambda I)^{-1/2}(L_K-L_{K,D})\|^2]\right)^{1/2}\|f_\lambda\|_K\right\} \nonumber\\
     &\leq&
     \lambda^{  {-1/2}} \mathcal Q_{|D|,\lambda}^{1/2}(\mathcal R_{|D|,\lambda}^{1/2}+\mathcal P_{|D|,\lambda}^{1/2}\|f_\lambda\|_K).
\end{eqnarray}
Since Assumption 5 holds with $1/2\leq r\leq 1$, it is easy to check
\begin{equation}\label{boundforflambda}
     \|f_\lambda\|_K=\|(L_K+\lambda I)^{-1}L_K^{1+r}h_\rho\|_K\leq\kappa^{2r-1}{ \|h_\rho\|_\rho}.
\end{equation}
Plugging (\ref{boundforflambda}) into (\ref{HK-variance-cl}) and inserting the obtain estimate together with (\ref{app-error-HK}) into \eqref{classic-error-dec-HK}, we obtain (\ref{error-dec-HK-prop}) directly. This completes the proof of Proposition \ref{Proposition:KRR-err-dec-HK}.
\end{proof}

Error decomposition under the $L_{\rho_X}^2$ norm is   more sophisticate, for strong mixing sequences.
Different from the previous decomposition under the $\mathcal H_K$ norm based on  (\ref{Op-KRR-d1})  and  the method  in \cite{Smale2007,Lin2017} based on (\ref{Op-KRR-d2}), our approach adopts   both (\ref{Op-KRR-d1}) and (\ref{Op-KRR-d2}). In fact, our analysis are divided into two stages: the first  is to bound $\|f_{D,\lambda}-f_\lambda\|_\rho$ by using (\ref{Op-KRR-d1}) and the other  is to estimate $\|f_{D,\lambda}-f_\lambda\|_K$ by utilizing (\ref{Op-KRR-d2}).

\begin{proposition}\label{Proposition:error-dec-KRR}
If Assumption 5 holds with $\frac12\leq r\leq 1$, then
\begin{eqnarray}\label{error-dec-KRR-l2-prop}
    &&E[\|f_{D,\lambda}-f_\rho\|_\rho]
    \leq
    \sqrt{2}\lambda^{-1}\mathcal P_{|D|,\lambda}^{1/2}\left(\mathcal T_{|D|}^{1/2}+\kappa^{2r-1}\|h_\rho\|_\rho\mathcal S_{|D|}^{1/2}\right)
    \nonumber\\
    &+&
    \kappa^{2r-1}\|h_\rho\|_\rho\mathcal P_{|D|,\lambda}^{1/2}+\mathcal R_{|D|,\lambda}^{1/2}+\lambda^r\|h_\rho\|_\rho.
\end{eqnarray}
\end{proposition}

\begin{proof}
It follows from the triangle inequality that
\begin{equation}\label{error-decompo-KRR}
   \|f_{D,\lambda}-f_\rho\|_\rho\leq \|f_{D,\lambda}-f_\lambda\|_\rho
   +\|f_{\lambda}-f_\rho\|_\rho.
\end{equation}
But  \cite{Smale2007} shows that under Assumption 5 with $1/2\leq r\leq 1$, there holds
\begin{equation}\label{Approximation-error}
  {  \|f_{\lambda}-f_\rho\|_\rho}\leq\lambda^r\|h_\rho\|_\rho.
\end{equation}
According to (\ref{Op-KRR-d1}), $\|f\|_\rho=\|L_K^{1/2}f\|_K$ for $f\in L_{\rho_X}^2$ and Schwarz inequality, we have
\begin{eqnarray}\label{l2error}
    &&E[\|f_{D,\lambda}-f_\lambda\|_\rho]
    \leq
    E[\|(L_K+\lambda I)^{-1}(L_K-L_{K,D})(f_{D,\lambda}-f_\lambda)\|_\rho] \nonumber\\
    &+&
    E[\|(L_K+\lambda I)^{-1}(L_K-L_{K,D}) f_\lambda\|_\rho] \nonumber\\
    &+&
    E[\|(L_K+\lambda I)^{-1}(S_D^T y_D-L_Kf_\rho)\|_\rho]\nonumber\\
    &\leq&
    \left(E[\|(L_K+\lambda I)^{-1/2}(L_K-L_{K,D})\|^2]\right)^{1/2}
    \left(E[\|f_{D,\lambda}-f_\lambda\|_K^2]\right)^{1/2}\nonumber\\
    &+&
    \left(E[\|(L_K+\lambda I)^{-1/2}(L_K-L_{K,D})\|^2]\right)^{1/2}\|f_\lambda\|_K \nonumber\\
    &+&
     \left(E[\|(L_K+\lambda I)^{-1/2}(S_D^T y_D-L_Kf_\rho)\|_K^2]\right)^{1/2}.
\end{eqnarray}
The only thing remainder is to bound $E[\|f_{D,\lambda}-f_\lambda\|_K^2]$.  We use the error decomposition strategy (\ref{Op-KRR-d2})  and obtain
\begin{eqnarray}\label{boundforKnorm}
    &&E[\|f_{D,\lambda}-f_\lambda\|^2_K] \\
    &\leq&
    2E[\|(L_{K,D}+\lambda I)^{-1}(S_D^T y_D-L_Kf_\rho)\|_K^2] \nonumber\\
    & +&
    2E[\|(L_{K,D}+\lambda I)^{-1}(L_K-L_{K,D})f_\lambda\|_K^2] \nonumber\\
    &\leq&
    2\lambda^{-2}(E[\|(S_D^T y_D-L_Kf_\rho)\|_K^2]+E[\|L_K-L_{K,D}\|^2]\|f_\lambda\|_K^2).\nonumber
\end{eqnarray}
Therefore, plugging  the above estimate and \eqref{boundforflambda} into (\ref{l2error}), it follows from (\ref{Def.P}), \eqref{Def.R}, \eqref{Def.S} and \eqref{Def.T} that
\begin{eqnarray*}
    &&E[\|f_{D,\lambda}-f_\lambda\|_\rho]
    \leq
    {  \sqrt{2}\lambda^{-1}\mathcal P_{|D|,\lambda}^{1/2}\left(\mathcal T_{|D|}^{1/2}+\kappa^{2r-1}\|h_\rho\|_\rho \mathcal S^{1/2}_{|D|}\right)}
    \\
    &+&
    \kappa^{2r-1}\|h_\rho\|_\rho\mathcal P_{|D|,\lambda}^{1/2}+\mathcal R_{|D|,\lambda}^{1/2}.
\end{eqnarray*}
This completes the proof of Proposition \ref{Proposition:error-dec-KRR} by noting $(a+b)^{1/2}\leq a^{1/2}+b^{1/2}$ for $a,b\geq 0$.
\end{proof}

\subsection{Error decomposition for DKRR}
Error decomposition for DKRR is well developed in the literature \cite{Chang2017} for i.i.d. samples, where the generalization error is divided into approximation error, sample error and distributed error.  The main tool to derive such an error decomposition is the  following covariance equality for  i.i.d. sequences
$$
       E[\langle f_{D_j,\lambda},  f_{D_k,\lambda}\rangle_\rho]
       =\langle E[f_{D_j,\lambda}],E[f_{D_j,\lambda}]\rangle,\ \forall k\neq j.
$$

Due to the dependence, the above covariance relation does not hold for strong mixing sequences.   Our strategy to conquer this challenge  is   to utilize different decomposition similar as (\ref{Op-KRR-d1}) and (\ref{Op-KRR-d2})  in two stages of analysis.
It follows from \eqref{operator DKRR} that
\begin{eqnarray*}
     &&\overline{f}_{D,\lambda}-f_\lambda\\
     &  = &
    \sum_{j=1}^m\frac{|D_j|}{|D|}(L_{K,D_j}+\lambda I)^{-1}S_{D_j}^Ty_{D_j}
    -
    \sum_{j=1}^m\frac{|D_j|}{|D|}(L_{K}+\lambda I)^{-1}S_{D_j}^Ty_{D_j}\\
    &+&
    \sum_{j=1}^m\frac{|D_j|}{|D|}(L_{K}+\lambda I)^{-1}S_{D_j}^Ty_{D_j}
    -(L_{K}+\lambda I)^{-1}L_Kf_\rho.
\end{eqnarray*}
Noting
$$
    \sum_{j=1}^m\frac{|D_j|}{|D|}(L_{K}+\lambda I)^{-1}S_{D_j}^Ty_{D_j}
    =
    (L_{K}+\lambda I)^{-1}S_D^Ty_D
$$
and
$$
    \sum_{j=1}^m\frac{|D_j|}{|D|}L_{K,D_j}f_\lambda=L_{K,D}f_\lambda,
$$
we  have
\begin{eqnarray}\label{Error:distl}
   &&\overline{f}_{D,\lambda}-f_\lambda \\
   & =&
    \sum_{j=1}^m\frac{|D_j|}{|D|}(L_{K}+\lambda I)^{-1}(L_K-L_{K,D_j})(f_{D_j,\lambda}-f_\lambda) \nonumber\\
    &+&
    (L_{K}+\lambda I)^{-1}(L_K-L_{K,D})f_\lambda
    +
    (L_{K}+\lambda I)^{-1}(S_D^Ty_D-L_Kf_\rho). \nonumber
\end{eqnarray}
Compared (\ref{Error:distl}) with  (\ref{Op-KRR-d1}), there is only a difference between the first term  $\sum_{i=1}^m\frac{|D_j|}{|D|}(L_{K}+\lambda I)^{-1}(L_K-L_{K,D_j})(f_{D_j,\lambda}-f_\lambda)$ and
$(L_K+\lambda I)^{-1}(L_K-L_{K,D})(f_{D,\lambda}-f_\lambda)$, which behaves similar as the distributed error in \cite{Chang2017} and provides a restriction to the number of local machines. Then, we derive the following error decomposition for DKRR, under both the $\mathcal H_K$ and $L_{\rho_X}^2$ metric.

 \begin{proposition}\label{Proposition:error-dec-DKRR}
If Assumption 5 holds with $\frac12\leq r\leq 1$, then
\begin{eqnarray}\label{Error-d-dkrr-prop}
    &&\max\{E[\|\overline{f}_{D,\lambda}-f_\rho\|_\rho],\lambda^{1/2}E[\|\overline{f}_{D,\lambda}-f_\rho\|_K]\}\nonumber\\
    &\leq&
    \lambda^r\|h_\rho\|_\rho+\sqrt{2}\lambda^{-1}\sum_{j=1}^m\frac{|D_j|}{|D|}\mathcal P_{|D_j|,\lambda}^{1/2}(\mathcal T^{1/2}_{|D_j|}+\kappa^{2r-1}\|h_\rho\|_\rho\mathcal S^{1/2}_{|D_j}|) \nonumber \\
    &+&
    \kappa^{2r-1}\|h_\rho\|_\rho\mathcal P_{|D|,\lambda}^{1/2}+\mathcal R_{|D|,\lambda}^{1/2}.
\end{eqnarray}
\end{proposition}

 \begin{proof} The triangle inequality yields
\begin{equation}\label{DKRR-dec1}
   \|\overline{f}_{D,\lambda}-f_\rho\|_*\leq
   \|\overline{f}_{D,\lambda}-f_\lambda\|_*+\|f_\lambda-f_\rho\|_*,
\end{equation}
 where $\|\cdot\|_*$ denotes either $\|\cdot\|_K$ or $\|\cdot\|_\rho$.
From  (\ref{Error:distl}), we have
\begin{eqnarray*}
   &&E[\|\overline{f}_{D,\lambda}-f_\lambda\|_*]\\
    &\leq&
    E\left[\sum_{j=1}^m\frac{|D_j|}{|D|}\|(L_{K}+\lambda I)^{-1}(L_K-L_{K,D_j})(f_{D_j,\lambda}-f_\lambda)\|_*\right]\\
    &+&
    E\left[\|(L_{K}+\lambda I)^{-1}(L_K-L_{K,D})f_\lambda\|_*\right]\\
   & +&
    E\left[\|(L_{K}+\lambda I)^{-1}(S_D^Ty_D-L_Kf_\rho)\|_*\right].
\end{eqnarray*}
Therefore,
\begin{eqnarray*}
    &&\max\left\{E[\|\overline{f}_{D,\lambda}-f_\lambda\|_\rho],
    \lambda^{1/2}E[\|\overline{f}_{D,\lambda}-f_\lambda\|_K]\right\}\\
    &\leq&
    \sum_{j=1}^m\frac{|D_j|}{|D|}\left(E[\|(L_{K}+\lambda I)^{-1/2}(L_K-L_{K,D_j})\|^2]\right)^{1/2}\\ &\times&\left(E[\|f_{D_j,\lambda}-f_\lambda\|_K^2]\right)^{1/2}\\
    &+&
    \left(E\left[\|(L_{K}+\lambda I)^{-1/2}(L_K-L_{K,D})\|^2 \right]\right)^{1/2}\|f_\lambda\|_K\\
    &+&
   \left( E\left[\|(L_{K}+\lambda I)^{-1/2}(S_D^Ty_D-L_Kf_\rho)\|_K^2\right]\right)^{1/2}.
\end{eqnarray*}
Plugging (\ref{boundforflambda}) and (\ref{boundforKnorm}) with $D$ being replaced by $D_j$ into the above estimate, we get
\begin{eqnarray*}
    &&\max\{E[\|\overline{f}_{D,\lambda}-f_\lambda\|_\rho],\lambda^{1/2}E[\|\overline{f}_{D,\lambda}-f_\lambda\|_K]\}\\
    &\leq&
    \sqrt{2}\lambda^{-1}\sum_{j=1}^m\frac{|D_j|}{|D|}\mathcal P_{|D_j|,\lambda}^{1/2}(\mathcal T_{|D_j|}^{1/2}+\kappa^{2r-1}\|h_\rho\|_\rho\mathcal S_{|D_j|}^{1/2})   \\
    &+&
    \kappa^{2r-1}\|h_\rho\|_\rho\mathcal P_{|D|,\lambda}^{1/2}+\mathcal R_{|D|,\lambda}^{1/2}.
\end{eqnarray*}
Inserting the above estimate, (\ref{Approximation-error}) and (\ref{app-error-HK}) into \eqref{DKRR-dec1}, we complete the proof of Proposition  \ref{Proposition:error-dec-DKRR}.
\end{proof}

\section{Proofs}\label{Sec.proofs}

In this section, we provide proofs of the main results.  We at first prove Theorem \ref{Theorem:KRR-err-dec-HK-1}.

\begin{proof}[Proof of Theorem \ref{Theorem:KRR-err-dec-HK-1}]
Since $D$ is a set of $\alpha$-mixing sequences with $\alpha$-mixing coefficient $\alpha_j$ and Assumptions   1-5 hold, it follows from  Lemma \ref{Lemma:Operator-difference}, Lemma \ref{Lemma:Operator-product} and Lemma \ref{Lemma:Operator-difference-2} respectively that
\begin{eqnarray}
   &\mathcal P_{|D|,\lambda}^{1/2}
    \leq
   \kappa\frac{\sqrt{\mathcal N(\lambda)} }{\sqrt{|D|}}
   +
   4 \kappa^{\frac{ 2\delta+2 }{\delta+2}}(\mathcal N(\lambda))^{\frac{1}{\delta+2}} \lambda^{-\frac{\delta}{2\delta+4}} \sqrt{\frac{1}{|D|}\sum_{\ell=1}^{|D|}(\alpha_\ell)^{\frac\delta{2+\delta}}},\label{b1}\\
    &\mathcal Q_{|D|,\lambda}^{1/2}\leq
      \frac{\sqrt{2}\kappa\sqrt{\mathcal N(\lambda)} }{  {\sqrt{|D|\lambda}}}+
    6 \kappa^{\frac{ 2\delta+2}{\delta+2}}(\mathcal N(\lambda))^{  {\frac{1}{\delta+2}}} \lambda^{-\frac{\delta+1}{\delta+2}} \sqrt{\frac1{|D|}\sum_{\ell=1}^{|D|} (\alpha_\ell)^{\frac\delta{2+\delta}}}
     +\sqrt{2},\nonumber\\
     &
     \mathcal R_{|D|,\lambda}^{1/2}
     \leq
     \frac{M\sqrt{\mathcal N(\lambda)} }{\sqrt{|D|}}
   +
   4 \kappa^{\frac{\delta}{\delta+2}}M(\mathcal N(\lambda))^{\frac{1}{\delta+2}} \lambda^{-\frac{\delta}{2\delta+4}}\sqrt{\frac1{|D|}\sum_{\ell=1}^{|D|}(\alpha_\ell)^{\frac\delta{2+\delta}}}.\label{b2}
\end{eqnarray}
Plugging above three estimates into \eqref{error-dec-HK-prop}, we obtain
\begin{eqnarray*}
    &&E[\|f_{D,\lambda}-f_\rho\|_K]
     \leq  \lambda^{r-1/2}\nonumber\\
    &+&
    c_1\lambda^{-1/2}\left(\frac{\sqrt{\mathcal N(\lambda)} }{\sqrt{|D|\lambda}}+
    (\mathcal N(\lambda))^{\frac{1}{\delta+2}} \lambda^{-\frac{\delta+1}{\delta+2}} \sqrt{\frac{1}{|D|}\sum_{\ell=1}^{|D|}(\alpha_\ell)^{\frac\delta{2+\delta}}}+1\right)\\
    &\times&
     \left(
    \frac{\sqrt{\mathcal N(\lambda)} }{\sqrt{|D|}}
   +
   (\mathcal N(\lambda))^{\frac{1}{\delta+2}} \lambda^{-\frac{\delta}{2\delta+4}}\sqrt{\frac{1}{|D|}\sum_{\ell=1}^{|D|}(\alpha_\ell)^{\frac\delta{2+\delta}}}\right).
\end{eqnarray*}
where
\begin{eqnarray*}
  &&c_1:=\max\left\{\sqrt{2}\kappa,6\kappa^\frac{2\delta +2}{ \delta+2},
   \sqrt{2}\right\}\\
   &\times&
      \max\left\{M+\kappa^{2r}\|h_\rho\|_\rho,4\|h_\rho\|_\rho\kappa^{2r-1+\frac{2\delta +2}{ \delta+2}}+4M\kappa^{\frac{2}{2+\delta}}\right\}.
\end{eqnarray*}
Due to Assumption 4 and $\lambda=|D|^{  {-1/(2r+s)}}$, we obtain
$$
    \frac{\sqrt{\mathcal N(\lambda)} }{\sqrt{|D|}}\leq \sqrt{C_0}|D|^\frac{-r}{2r+s},
$$
and
$$
   (\mathcal N(\lambda))^{\frac{1}{\delta+2}} \lambda^{-\frac{\delta}{2\delta+4}}
   \leq (\mathcal N(\lambda))^{\frac{1}{\delta+2}} \lambda^{-\frac{\delta+1}{\delta+2}}
   \leq
   C_0^\frac{1}{2+\delta}|D|^\frac{(s+1)(\delta+1)}{(2r+s)(\delta+2)}.
$$
Combining the above three estimates and noting $r\geq 1/2$, we have from $(a+b)^{1/2}\leq a^{1/2}+b^{1/2}$ for $a,b>0$ that
\begin{eqnarray*}
    E[\|f_{D,\lambda}-f_\rho\|_K]
    &\leq& c_2 \left(1+|D|^{-\frac12+\frac{(s+1)(\delta+1)}{(2r+s)(\delta+2)}}
    \sum_{\ell=1}^{|D|}(\alpha_\ell)^{\frac\delta{4+2\delta}}\right) \\
    &\times&
    \left( |D|^{-\frac{r-1/2}{2r+s}}+|D|^{-\frac12+\frac{2(s+1)(\delta+1)+\delta+2 }{(2r+s)(2\delta+4)}}
    \sum_{\ell=1}^{|D|}(\alpha_\ell)^{\frac\delta{4+2\delta}}\right),
\end{eqnarray*}
where
$c_2:=c_1(\sqrt{C_0}+1)(\sqrt{C_0}+C_0^{1/(2+\delta)})$. This completes the proof of Theorem \ref{Theorem:KRR-err-dec-HK-1} with $C=c_2$.
\end{proof}

\begin{proof}[Proof of Corollary \ref{Corollary:KRR-err-dec-HK-1}]
It follows from (\ref{def-Galpha11}) that
\begin{eqnarray}\label{b3.3}
    &&\sum_{\ell=1}^{|D|}(a_\ell)^{\frac\delta{4+2\delta}}
    \leq
    (c_1^*)^{\frac{\delta}{4+2\delta}}\sum_{\ell=1}^{|D|}\ell^{-\frac{\gamma_1\delta}{4+2\delta}} \nonumber\\
    &\leq&
    (c_1^*)^{\frac{\delta}{4+2\delta}}\frac{4+2\delta}{|4+2\delta-\gamma_1\delta|}
    \left|1-{  |D|^{\frac{4+2\delta-\gamma_1\delta}{4+2\delta}}}\right|.
\end{eqnarray}
Inserting the above estimate into the righthand side of (\ref{KRR-err-HK-1}) and setting
$
    \delta=\frac{4(2r+s)\varepsilon}{2+s-2(2r+s)\varepsilon},
$
we obtain from
 (\ref{suff-cond-1}) that
 $$
   E[\|f_{D,\lambda}-f_\rho\|_K]
   \leq C'|D|^{-\frac{r-1/2}{2r+s}+\varepsilon},
 $$
where $C':=\left(1+(c_1^*)^{\frac{\delta}{4+2\delta}}\frac{4+2\delta}{|4+2\delta-\gamma_1\delta|}\right)^2C.
$
This completes the proof of Corollary \ref{Corollary:KRR-err-dec-HK-1}.
\end{proof}

\begin{proof}[Proof of Theorem \ref{Theorem:KRR-err-dec-rho-1}]
It follows from Lemma \ref{Lemma:Operator-difference} and Lemma \ref{Lemma:Operator-difference-2} that
\begin{eqnarray}\label{media1}
 &&\left(\mathcal T_{|D|}^{1/2}+\kappa^{2r-1}\|h_\rho\|_\rho\mathcal S^{1/2}_{|D|}\right) \nonumber\\
 &\leq&
 \frac{\kappa}{\sqrt{|D|}}\left(1+6\sum_{\ell=1}^{|D|}\sqrt{\alpha_\ell}\right)
 (M+\kappa^{2r}\|h_\rho\|_\rho).
\end{eqnarray}
 Since $D$ is a set of $\alpha$-mixing sequences with $\alpha$-mixing coefficient $\alpha_j$,
Assumptions 1-5 hold, plugging (\ref{media1}), (\ref{b1}) and (\ref{b2}) into  \eqref{error-dec-KRR-l2-prop}, we obtain
\begin{eqnarray*}
    &&E[\|f_{D,\lambda}-f_\rho\|_\rho]
    \leq \lambda^r\|h_\rho\|_\rho +\left(
   \frac{1}{\lambda\sqrt{|D|}} \left(1+\sum_{\ell=1}^{|D|}\sqrt{\alpha_\ell}\right)+1\right)
    \\
    &\times&
     c_2\left(\frac{\sqrt{\mathcal N(\lambda)} }{\sqrt{|D|}}
   +
   (\mathcal N(\lambda))^{\frac{1}{\delta+2}} \lambda^{-\frac{\delta}{2\delta+4}} \frac{1}{\sqrt{|D|}}\sum_{\ell=1}^{|D|}(\alpha_\ell)^{\frac\delta{4+2\delta}}\right),
\end{eqnarray*}
where
$$
   c_2:=6(4\kappa^{\frac{ 2(\delta+1)}{ \delta+2}}+\kappa)(\sqrt{2}\kappa(M+\kappa^{2r}\|h_\rho\|_\rho)+{  \kappa^{2r-1}}\|h_\rho\|_\rho)(M+M^2\kappa^{\frac{1}{2\delta+1}}).
$$
Due to Assumption 4, $\lambda=|D|^{-1/(2r+s)}$ and $2r+s\geq 2$, we then have
\begin{eqnarray*}
   &&E[\|f_{D,\lambda}-f_\rho\|_\rho]\\
   &\leq &
   c_3|D|^{-\frac{r}{2r+s}} \left(1+\sum_{\ell=1}^{|D|}\sqrt{\alpha_\ell}\right)\left( 1 +|D|^{-\frac12+\frac{2s+\delta+2\delta r+4r}{(2r+s)(2\delta+4)}}\sum_{\ell=1}^{|D|}(\alpha_\ell)^{\frac\delta{4+2\delta}}\right),
\end{eqnarray*}
where $c_3:=\max\{\|h_\rho\|_\rho,2c_2(\sqrt{C_0}+C_0^{\delta/(4+2\delta)})$. This completes the proof of Theorem \ref{Theorem:KRR-err-dec-rho-1} with $\hat{C}=c_3$.
\end{proof}

\begin{proof}[Proof of Corollary \ref{Corollary:KRR-err-dec-rho-1}]
Since i.i.d. data always implies (\ref{def-Galpha11}), the lower bound of (\ref{Upper-KRR-rho}) can be found in \cite[Theorem 3]{Caponnetto2007}.
The proof of the upper bound of (\ref{Upper-KRR-rho})  is similar as that of Corollary \ref{Corollary:KRR-err-dec-HK-1}. Setting $\delta=\frac{8r\varepsilon+4s\varepsilon}{1-s-4r\varepsilon-2s\varepsilon}$, (\ref{Upper-KRR-rho})
 follows from  \eqref{b3.3} and (\ref{suff-cond-2}) with
 $\hat{C}':= \hat{C}\left(2+(c_1^*)^{\frac{\delta}{4+2\delta}}\frac{4+2\delta}{|4+2\delta-\gamma_1\delta|}\right)^2
$. This completes the proof of Corollary \ref{Corollary:KRR-err-dec-rho-1}.
\end{proof}

\begin{proof}[Proof of Theorem \ref{Theorem:KRR-derr-dec}]
Under Assumptions 1-5, plugging  (\ref{media1}) with $D$ being replaced by $D_j$, (\ref{b1}) and (\ref{b2}) into (\ref{Error-d-dkrr-prop}), we get
\begin{eqnarray*}
    &&\max\{E[\|\overline{f}_{D,\lambda}-f_\rho\|_\rho],\lambda^{1/2}E[\|\overline{f}_{D,\lambda}-f_\rho\|_K]\} \\
    &\leq&
   \lambda^r\|h_\rho\|_\rho+ c_4\sum_{j=1}^m\frac{|D_j|}{|D|}\frac{1}{\lambda\sqrt{|D_j|}} \left(1+\sum_{\ell=1}^{|D_j|}\sqrt{\alpha_\ell}\right)\\
    &\times&
    \left(\frac{\sqrt{\mathcal N(\lambda)} }{\sqrt{|D_j|}}
   +
   (\mathcal N(\lambda))^{\frac{1}{\delta+2}} \lambda^{-\frac{\delta}{2\delta+4}} \frac{1}{\sqrt{|D_j|}}
   \sum_{\ell=1}^{|D_j|}(\alpha_\ell)^{\frac\delta{4+2\delta}}\right)\\
   &+&
   c_5\left(
    \frac{\sqrt{\mathcal N(\lambda)} }{\sqrt{|D|}}
   +
   (\mathcal N(\lambda))^{\frac{1}{\delta+2}} \lambda^{-\frac{\delta}{2\delta+4}}\frac{1}{\sqrt{|D|}}\sum_{\ell=1}^{|D|}(\alpha_\ell)^{\frac\delta{4+2\delta}}\right),
\end{eqnarray*}
where $c_4:=6\sqrt{2}(\kappa+\kappa^{\frac{ 2\delta +2}{ \delta+2}})\kappa(M+\kappa^{2r})\|h_\rho\|_\rho
$ and $c_5=\kappa^{2r}\|h_\rho\|_\rho(\kappa+\kappa^{\frac{ 2\delta +2}{ \delta+2}})+M+4\kappa^{\frac{2}{2+\delta}}$.
Due to Assumption 4, $\lambda=|D|^{-1/(2r+s)}$ and $2r+s\geq 2$, we then have
\begin{eqnarray*}
    &&\max\{E[\|\overline{f}_{D,\lambda}-f_\rho\|_\rho],\lambda^{1/2}E[\|\overline{f}_{D,\lambda}-f_\rho\|_K]\} \\
    &\leq&
    c_6|D|^{-r/(2r+s)}+c_6|D|^{1/(2r+s)}
    \sum_{j=1}^m\frac{|D_j|}{|D|}|D_j|^{-1/2} \left(1+\sum_{\ell=1}^{|D_j|}\sqrt{\alpha_\ell}\right)\\
    &\times&
     \left(|D_j|^{-\frac12}|D|^\frac{s}{4r+2s} +|D|^\frac{2s+\delta+1}{(2r+s)(2\delta+4)}\frac{1}{\sqrt{|D_j|}}\sum_{\ell=1}^{|D_j|}(\alpha_\ell)^{\frac\delta{4+2\delta}}\right) \\
     &+&
     c_6 \left( |D|^{-\frac{r}{2r+s}}  +|D|^{-\frac12+\frac{2s+\delta+1}{(2r+s)(2\delta+4)}}\sum_{\ell=1}^{|D|}(\alpha_\ell)^{\frac\delta{4+2\delta}}\right),
\end{eqnarray*}
where $c_6:=\max\{\|h_\rho\|_\rho,(\sqrt{C_0}+C_0^{\delta/(4+2\delta)}) (c_4+c_5)\}$. This completes the proof of Theorem \ref{Theorem:KRR-derr-dec}.
\end{proof}

\begin{proof}[Proof of Corollary \ref{Corollary:KRR-derr-dec}]
  Setting $\delta=\frac{12r+6s}{(\gamma_1-2)(2r+s)-2r-1}$, it follows from (\ref{KRR-derr-rho-1}), (\ref{suff-cond-2}), \eqref{b3.3}, $|D_1|=\dots=|D_m|$ and (\ref{condition-m}) that
  \eqref{DKRR-op-1} and \eqref{DKRR-op-2} hold with
 $\hat{C}':= \bar{C}\left(1+(c_1*)^{\frac{\delta}{4+2\delta}}\frac{4+2\delta}{|4+2\delta-\gamma_1\delta|}\right)^2
$. This completes the proof of  Corollary \ref{Corollary:KRR-derr-dec}.
\end{proof}

\section*{Appendix: Proof of   Lemmas \ref{Lemma:Operator-difference}, \ref{Lemma:Operator-product}, \ref{Lemma:Operator-difference-2}}
In the Appendix, we present proofs for Lemmas   \ref{Lemma:Operator-difference}, \ref{Lemma:Operator-product}, \ref{Lemma:Operator-difference-2}. The main tools are Lemma \ref{Lemma:tool1} and Lemma \ref{Lemma:mixing-independent}, which show  the $\alpha$-mixing coefficients of Banach valued sequences $\{K_{x_i}\}$ is the same as those of $\{x_i\}$.

\begin{proof}[Proof of Lemma \ref{Lemma:Operator-difference}] The proof of (\ref{operator-d1}) can be found in
\cite[Lemma 5.1]{Sun2010}. We only need to prove (\ref{operator-d2}). Defined
\begin{equation}\label{eta1}
    \eta_1 (x) =\left(L_K+\lambda I\right)^{-1/2} \langle \cdot,
     K_{x_i}\rangle_K K_{x}, \qquad x\in {\mathcal X}.
\end{equation}
It takes values in $HS({\mathcal H}_K)$, the Hilbert space of
Hilbert-Schmidt operators on ${\mathcal H}_K$, with inner product
$\langle A, B\rangle_{HS} = \hbox{Tr}(B^T A).$   The norm is given by
$\|A\|_{HS}^2 =\sum_{i} \|A e_i\|_K^2$ where $\{e_i\}$ is an
orthonormal basis of ${\mathcal H}_K$. The space $HS({\cal H}_K)$ is
a subspace of the space of bounded linear operators on ${\cal H}_K$,
denoted as $(L({\cal H}_K), \|\cdot\|)$, with the norm relations
\begin{equation}\label{normrelation}
 \|A\| \leq \|A\|_{HS},
\qquad \|A B \|_{HS} \leq \|A\|_{HS} \|B\|.
\end{equation}
Then, it can be found in \cite[eqs.(52)]{Lin2017} that
\begin{equation}\label{Eeta1}
    E\left[\|\eta_1\|^2_{HS}\right]  \leq
  \kappa^2 {\mathcal N}(\lambda).
\end{equation}
Now, we
apply Lemma \ref{Lemma:tool1} to the random variable $\eta_1$ to prove (\ref{operator-d2}).
The random variable $\eta_1(x)$ defined by (\ref{eta1}) has
mean ${ E[\eta_1]} = \left(L_K+\lambda I\right)^{-1/2} L_K$ and sample
mean $\left(L_K+\lambda I\right)^{-1/2} L_{K, D}$.
Then,
$$
       \left(L_K+\lambda I\right)^{-1/2} (L_K-L_{K,D})=E[\eta_1]-\frac{1}{|D|}\sum_{x_i\in D}\eta_1(x_i).
$$
For any $i$,
we have from (\ref{Eeta1}) that
\begin{equation}\label{equal-i}
     E_{x_i}\langle \eta_1(x_i),\eta_1(x_i)\rangle_{HS}=E_{x_i}[\|\eta_1(x_i)\|^2_{HS}]\leq
     \kappa^2\mathcal N(\lambda).
\end{equation}
For any $i\neq j$ and $\delta>0$, it follows from Lemma \ref{Lemma:tool1} with $u=v=\delta+2$ and $t=(\delta+2)/\delta$ that
\begin{eqnarray}\label{prove-1-1.1.a}
   &&\left| E_{x_i,x_j}[\langle \eta_1(x_i),\eta_1(x_j)\rangle_{HS}]
   -\|(L_K+\lambda I)^{-1/2}L_K\|_{HS}^2\right| \nonumber\\
   &\leq &
   15\left(\alpha(\mathcal M_{1,\min\{i,j\}},\mathcal M_{\max\{i,j\},\infty})\right)^\frac{\delta}{2+\delta}\|\eta_1(x_i)\|_{2+\delta}^2
\end{eqnarray}
Noting that for arbitrary $x\in\mathcal X$,  $(L_{K}+\lambda I)^{-1/2}$ and $K_{x_i}\otimes K_{x_i}:= \langle \cdot,
     K_{x_i}\rangle_K K_{x}$ are positive operators, we have
$$
   \|(L_{K}+\lambda I)^{-1/2}K_{x_i}\otimes K_{x_i}\|_{HS}=\|K_{x_i}\otimes K_{x_i}(L_{K}+\lambda I)^{-1/2}\|_{HS}.
$$
Then, it follows from \eqref{normrelation} and
$$
     \sup_{x\in\mathcal X}\|K_x\otimes K_x\|_{HS}= \sup_{x\in \mathcal X}\|K_x\|_K^2=\sup_{x\in\mathcal X} K(x,x)=\kappa^2
$$
that
 for arbitrary $x\in\mathcal X$, there holds
$$
     \|\eta_1(x)\|_{HS}
     \leq \|(L_{K}+\lambda I)^{-1/2}\| \|K_{x_i}\otimes K_{x_i}\|_{HS}
     \leq \lambda^{-1/2}\kappa^2.
$$
Hence, (\ref{equal-i}) and the above estimate yield
\begin{eqnarray*}
    &&\|\eta_1(x_i)\|_{2+\delta}^{2+\delta}
    =
    E_{x_i}[ \|\eta_1(x_i)\|_{HS}^{2+\delta}]\nonumber\\
    &\leq& \sup_{x\in\mathcal X}  \|\eta_1(x)\|_{HS}^\delta
    E_{x_i}[\|\|\eta_1(x_i)\|_{HS}^{2}]
     \leq
    \kappa^{2\delta+2}\mathcal N(\lambda) \lambda^{-\delta/2}.
\end{eqnarray*}
Plugging this into (\ref{prove-1-1.1.a}) and noting (\ref{def-alpha}), we have
\begin{eqnarray}\label{inequal-i,j,abc}
   &&E_{x_i,x_j}[\langle \eta_1(x_i),\eta_1(x_j)\rangle_{HS}]\\
   &\leq&
   \|(L_K+\lambda I)^{-1/2}L_K\|_{HS}^2+15\kappa^{\frac{{ 4(\delta+1)}}{\delta+2}}(\alpha_{|j-i|})^{\frac\delta{2+\delta}}
   (\mathcal N(\lambda))^{\frac{2}{\delta+2}} \lambda^{-\frac{\delta}{\delta+2}}.     \nonumber
\end{eqnarray}
Then,
\begin{eqnarray*}
  &&E[\|\left(L_K+\lambda I\right)^{-1/2} (L_K-L_{K,D})\|_{HS}^2] \\
  &=&
  E\left[\left\langle E[\eta_1]-\frac{1}{|D|}\sum_{x_i\in D}\eta_1(x_i),
  E[\eta_1]-\frac{1}{|D|}\sum_{x_j\in D}\eta_1(x_j)\right\rangle_{HS}  \right]\\
  &=&
  \|(L_{K}+\lambda I)^{-1/2}L_K\|_{HS}^2
  -2E\left[\left\langle  E[\eta_1],\frac{1}{|D|}\sum_{x_i\in D}\eta_1(x_i)\right\rangle_{HS}\right]\\
  &+&
  E\left[\left\langle \frac{1}{|D|}\sum_{x_i\in D}\eta_1(x_i),\frac{1}{|D|}\sum_{x_j\in D}\eta_1(x_j)\right\rangle_{HS}\right]\\
  &=&
  E\left[\left\langle \frac{1}{|D|}\sum_{x_i\in D}\eta_1(x_i),\frac{1}{|D|}\sum_{x_j\in D}\eta_1(x_j)\right\rangle_{HS}\right]\\
  &-&\|(L_{K}+\lambda I)^{-1/2}L_K\|_{HS}^2.
\end{eqnarray*}
But (\ref{equal-i}) and (\ref{inequal-i,j,abc}) yield
\begin{eqnarray*}
   &&E\left[\left\langle \frac{1}{|D|}\sum_{x_i\in D}\eta_1(x_i),\frac{1}{|D|}\sum_{x_j\in D}\eta_1(x_j)\right\rangle_{HS}\right]\\
   &=&
   E\left[\sum_{x_i\in D}\left\langle \frac{1}{|D|}\eta_1(x_i),\frac{1}{|D|} \eta_1(x_i)\right\rangle_{HS}\right]\\
   &+&
   E\left[\left\langle \frac{1}{|D|}\sum_{x_i\in D}\eta_1(x_i),\frac{1}{|D|}\sum_{j\neq i}{  \eta_1(x_j)}\right\rangle_{HS}\right]\\
   &=&
   \frac1{|D|^2}\sum_{x_i\in D}E\left[\langle \eta_1(x_i),\eta_1(x_i)\rangle_{HS}\right]\\
   &+&
   \frac1{|D|^2}\sum_{x_i\in D}\sum_{j\neq i} E\left[\langle \eta_1(x_i),\eta_1(x_j)\rangle_{HS}\right]\\
   &\leq&
   \frac{\kappa^2\mathcal N(\lambda) }{|D|}+
   \|(L_K+\lambda I)^{-1/2}L_K\|_{HS}^2\\
   &+&
    15 \kappa^{\frac{4(\delta+1)}{\delta+2}}(\mathcal N(\lambda))^{\frac{2}{\delta+2}} \lambda^{-\frac{\delta}{\delta+2}} \frac{1}{{ |D|^2}}\sum_{i=1}^{|D|}\sum_{j\neq i}(\alpha_{|j-i|})^{\frac\delta{2+\delta}}.
\end{eqnarray*}
We then have
\begin{eqnarray*}
   &&E[\|\left(L_K+\lambda I\right)^{-1/2} (L_K-L_{K,D})\|_{HS}^2]\\
   &\leq&
   \frac{\kappa^2\mathcal N(\lambda) }{|D|}
   +
   15 \kappa^{\frac{4(\delta+1)}{ \delta+2}}(\mathcal N(\lambda))^{\frac{2}{\delta+2}} \lambda^{-\frac{\delta}{\delta+2}} \frac1{|D|}\sum_{i=1}^{|D|}(\alpha_i)^{\frac\delta{2+\delta}}.
\end{eqnarray*}
Then Lemma \ref{Lemma:Operator-difference} follows from (\ref{normrelation}).
\end{proof}

\begin{proof}[Proof of Lemma \ref{Lemma:Operator-product}]
For invertible positive operators $A,B$, we have
\begin{eqnarray*}
    A^{-1}B&=&(A^{-1}-B^{-1})B+I
    = A^{-1}(B-A)+I\\
    &=&(A^{-1}-B^{-1})(B-A)+B^{-1}(B-A)+I\\
    &=&A^{-1}(B-A)B^{-1}(B-A)+B^{-1}(B-A)+I.
\end{eqnarray*}
Set $A=(L_{K,D}+\lambda I)$ and $B=(L_{K}+\lambda I).$ We then have
\begin{eqnarray*}
    &&\|(L_{K,D}+\lambda I)^{-1}(L_{K}+\lambda I)\|\\
    &\leq&
    \frac1\lambda\|(L_K-L_{K,D})(L_K+\lambda I)^{-1}(L_K-L_{K,D})\|\\
    &+& \frac1{\sqrt{\lambda}}\|(L_K+\lambda I)^{-1/2}(L_K-L_{K,D})\|+1\\
    &=&
    \frac1\lambda \| (L_K+\lambda I)^{-1/2}(L_K-L_{K,D})\|^2\\
    &+&
    \frac1{\sqrt{\lambda}}\| (L_K+\lambda I)^{-1/2}(L_K-L_{K,D})\|+1.
\end{eqnarray*}
Taking expectation and using the Schwarz inequality, we obtain
\begin{eqnarray*}
    &&E\left[\|(L_{K,D}+\lambda I)^{-1}(L_{K}+\lambda I)\|\right]\\
    \leq
    &&\frac1\lambda E\left[\| (L_K+\lambda I)^{-1/2}(L_K-L_{K,D})\|^2\right]\\
    &+&
    \frac1{\sqrt{\lambda}}\left(E\left[\| (L_K+\lambda I)^{-1/2}(L_K-L_{K,D})\|^2\right]\right)^{1/2}+1.
\end{eqnarray*}
 Then, it follows from Lemma \ref{Lemma:Operator-difference}, $(a+b)^{1/2}\leq a^{1/2}+b^{1/2}$ and $\sqrt{a}\leq a+1$ for $a,b>0$ that
 \begin{eqnarray*}
    &&E\left[\|(L_{K,D}+\lambda I)^{-1}(L_{K}+\lambda I)\|\right]\\
    &\leq&
    \frac{\kappa^2\mathcal N(\lambda) }{|D|\lambda}+
    15 \kappa^{\frac{4(\delta+1)}{ \delta+2}}(\mathcal N(\lambda))^{\frac{2}{\delta+2}} \lambda^{-\frac{2\delta+2}{\delta+2}} \frac{1}{|D|}\sum_{i=1 }^{|D|}(\alpha_i)^{\frac\delta{2+\delta}}\\
    &+&
    \frac{\kappa \sqrt{\mathcal N(\lambda)} }{\sqrt{\lambda|D|}}
    +
     \sqrt{15} \kappa^{\frac{2(\delta+1)}{\delta+2}}(\mathcal N(\lambda))^{\frac{1}{\delta+2}} \lambda^{-\frac{\delta+1}{\delta+2}}  \sqrt{\frac1{|D|}\sum_{i= 1}^{|D|}(\alpha_i)^{\frac\delta{2+\delta}}}+1\\
     &\leq&
     \frac{2\kappa^2\mathcal N(\lambda) }{|D|\lambda}+30 \kappa^{\frac{4(\delta+1)}{\delta+2}}(\mathcal N(\lambda))^{\frac{2}{\delta+2}} \lambda^{-\frac{2\delta+2}{\delta+2}} \frac1{|D|}\sum_{i=1}^{|D|}(\alpha_i)^{\frac\delta{2+\delta}}
     +2.
\end{eqnarray*}
This completes the proof of Lemma \ref{Lemma:Operator-product}.
\end{proof}

\begin{proof}[Proof of Lemma \ref{Lemma:Operator-difference-2}]
Define  $\eta_2(z)=yK_{x}$ for $z=(x, y)\in {\mathcal Z}$. Then, it is easy to check
$$
   E[\eta_2]= L_Kf_\rho, \qquad \mbox{and}\quad
   \frac{1}{|D|}{  \sum_{z_i\in \mathcal Z}}\eta_2(z_i)=S_D^Ty_D.
$$
But
$$
     \|\eta_2(z)\|_K\leq \kappa M,\qquad\mbox{and}\quad
     { E_z[\|\eta_2(z)\|^2_K] \leq M^2\kappa^2}.
$$
Then we obtain from Lemma \ref{Lemma:tool1} with $u=v=\infty$ and $t=1$ that
$$
     E\left[\langle y_iK_{x_i},y_jK_{x_j}\rangle_K\right]
     \leq
     \|L_Kf_\rho\|^2_K+{  15\alpha_{|i-j|}}\kappa^2 M^2.
$$
This together with the same method as that in the proof of Lemma \ref{Lemma:Operator-difference} implies (\ref{operator-d-2-1}). Now we trun to prove  (\ref{operator-d-2-2}).
Define
\begin{equation}\label{eta3}
   \eta_3(z) =\left(L_K+\lambda I\right)^{-1/2}\left(yK_{x}\right), \qquad z=(x, y)\in {\mathcal Z}.
\end{equation}
It takes values in ${\mathcal H}_K$ and satisfies
$$
   E[\eta_3]=(L_K+\lambda I)^{-1/2}L_Kf_\rho, \quad
   \frac{1}{|D|}{  \sum_{z_i\in \mathcal Z}}\eta_3(z_i)=(L_{K,D}+\lambda I)^{-1/2}S_D^Ty_D.
$$
Furthermore, it can be found in \cite[P.28]{Lin2017} that
\begin{equation}\label{proof.3.1}
   \left\|\eta_3 (z)\right\|_K \leq \frac{\kappa M}{\sqrt{\lambda}},  \qquad
     E_{z}\left[\|\eta_3(z)\|^2_{K}\right]  \leq M^2{\mathcal N}(\lambda).
\end{equation}
 For any $i$,
we have from (\ref{proof.3.1}) that
\begin{equation}\label{equal-i-for-3}
     E_{z_i}\langle \eta_3(z_i),\eta_3(z_i)\rangle_K=E_{z_i}[\|\eta_3(z_i)\|^2_K]\leq
     M^2{\mathcal N}(\lambda).
\end{equation}
For any $i\neq j$ and $\delta>0$, it follows from Lemma \ref{Lemma:tool1} with $u=v=\delta+2$ and $t=(\delta+2)/\delta$ that
\begin{eqnarray}\label{prove-1-1.1}
   &&\left| E_{z_i,z_j}[\langle \eta_3(z_i),\eta_3(z_j)\rangle_K]
   -\|(L_K+\lambda I)^{-1/2}L_Kf_\rho\|_K^2\right| \nonumber\\
   &\leq &
   15\left(\alpha(\mathcal M_{1,{\min\{i,j\}}},\mathcal M_{\max\{i,j\},\infty})\right)^\frac{\delta}{2+\delta}\|\eta_3(z_i)\|_{2+\delta}^2,
\end{eqnarray}
where
\begin{eqnarray*}
    &&\|\eta_3(z_i)\|_{2+\delta}^{2+\delta}
    =E_{z_i}[ \|\eta_3(z_i)\|_K^{2+\delta}]\\
    &\leq& \sup_{z\in\mathcal Z}  \|\eta_3(z)\|_{K}^\delta
    E_{z_i}[ \|\eta_3(z_i)\|_K^{2}]
    \leq
    \kappa^{\delta}M^{\delta+2}\mathcal N(\lambda) \lambda^{-\delta/2}.
\end{eqnarray*}
Then,
\begin{eqnarray}\label{inequal-i,j}
   &&E_{z_i,z_j}[\langle \eta_3(z_i),\eta_3(z_j)\rangle_K]
    \leq
   \|(L_K+\lambda I)^{-1/2}L_Kf_\rho\|_{K}^2\nonumber\\
   &+&
   15 \kappa^{\frac{2\delta}{\delta+2}}M^2(\alpha_{|j-i|})^{\frac\delta{2+\delta}}
   (\mathcal N(\lambda))^{\frac{2}{\delta+2}} \lambda^{-\frac{\delta}{\delta+2}}.
\end{eqnarray}
Therefore, we can use the same method as that in the proof of Lemma \ref{Lemma:Operator-difference} and obtain
\begin{eqnarray*}
   &&E[\|\left(L_K+\lambda I\right)^{-1/2} (L_Kf_\rho-S_D^Ty_D)\|_K^2]\\
   &\leq&
   \frac{M^2\mathcal N(\lambda) }{|D|}
   +
   15 \kappa^{\frac{2\delta}{\delta+2}}M^2(\mathcal N(\lambda))^{\frac{2}{\delta+2}} \lambda^{-\frac{\delta}{\delta+2}} \frac{1}{|D|} \sum_{i=1}^{|D|} (\alpha_i)^{\frac\delta{2+\delta}}.
\end{eqnarray*}
The proof of Lemma \ref{Lemma:Operator-difference-2} is completed.
\end{proof}

\section*{Acknowledgement}
The authors would like to thank AE, two anonymous referees  and Professor Zheng-Chu Guo in Zhejiang University for their constructive
suggestions. The work  is supported
 by   National Key R\&D Program of China (No.2020YFA0713900) and   National Natural Science Foundation of China (No. 61876133).

\end{document}